\documentclass{article}

\usepackage[utf8]{inputenc} 
\usepackage[T1]{fontenc}    
\usepackage{url}            
\usepackage{booktabs}       
\usepackage{amsfonts}       
\usepackage{nicefrac}       
\usepackage{microtype}      
\usepackage[dvipsnames]{xcolor}         

\usepackage{amsthm}         
\usepackage{amsmath,amssymb}
\usepackage{comment}        
\usepackage{enumitem}  
\usepackage{cite} 
\usepackage{parskip}
\usepackage[margin=1.2in]{geometry} 
\usepackage[ruled, linesnumbered]{algorithm2e}
\usepackage{tabulary,booktabs}

\newtheorem{definition}{Definition}[section]
\newtheorem{theorem}{Theorem}[section]
\newtheorem*{theorem*}{Theorem}
\newtheorem{proposition}{Proposition}[section]

\newtheorem{lemma}[theorem]{Lemma}
\newtheorem{assumption}{Assumption}[section]

\newcommand{\lrgd}{\texttt{LRGD}}

\newcommand{\gd}{\texttt{GD}}

\newcommand{\op}{\mathrm{op}}

\newcommand{\spn}{\mathrm{span}}

\newcommand{\Bcal}{\mathcal{B}}

\newcommand{\Fcal}{\mathcal{F}}

\newcommand{\Rbb}{\mathbb{R}}
\newcommand{\R}{\mathbb{R}}

\newcommand{\thetab}{\bar\theta}

\usepackage{bbm}

\definecolor{dark_red}{rgb}{0.2,0,0}

\newcommand{\gr}{\nabla}
\newcommand{\grhat}{\hat{\nabla}}

\usepackage{mysymbol}
\usepackage{mathtools}

\renewcommand\footnotemark{}
\usepackage{tcolorbox}

\usepackage[hidelinks]{hyperref}
\definecolor{darkred}{RGB}{150,0,0}
\definecolor{darkgreen}{RGB}{0,150,0}
\definecolor{darkblue}{RGB}{0,0,150}
\hypersetup{colorlinks=true, linkcolor=darkred, citecolor=darkblue, urlcolor=darkblue}

\mathtoolsset{showonlyrefs}

\usepackage{titling}

\title{\textbf{Gradient Descent for Low-Rank Functions}\vspace{1em}}

\author{%
  Romain~Cosson\\
  \small Department of EECS\\
  \small Massachusetts Institute of Technology\\
  \small \texttt{cosson@mit.edu}\\
  \and
  Ali~Jadbabaie\\
  \small Department of CEE\\
  \small Massachusetts Institute of Technology\\
  \small \texttt{jadbabai@mit.edu}\\
  \and
  Anuran~Makur\\
  \small Department of CS and School of ECE\\
  \small Purdue University\\
  \small \texttt{amakur@purdue.edu}\\
  \and
  Amirhossein~Reisizadeh\\
  \small Department of EECS\\
  \small Massachusetts Institute of Technology\\
  \small \texttt{amirr@mit.edu}\\
  \and
  Devavrat~Shah\thanks{The author ordering is alphabetical. This research was supported in part the MIT-IBM Watson AI Lab and in part by the Department of Defense (DOD) Office of Naval Research (ONR) under Grant N00014-20-1-2394.}\\
  \small Department of EECS\\
  \small Massachusetts Institute of Technology\\
  \small \texttt{devavrat@mit.edu}
}
\date{}

\begin{document}

\maketitle

\begin{abstract} 
Several recent empirical studies demonstrate that important machine learning tasks, e.g., training deep neural networks, exhibit low-rank structure, where the loss function varies significantly in only a few directions of the input space. In this paper, we leverage such low-rank structure to reduce the high computational cost of canonical gradient-based methods such as gradient descent ({\gd}). Our proposed \emph{Low-Rank Gradient Descent} ({\lrgd}) algorithm finds an $\epsilon$-approximate stationary point of a $p$-dimensional function by first identifying $r \leq p$ significant directions, and then estimating 
the true $p$-dimensional gradient at every iteration by computing directional derivatives only 
along those $r$ directions. We establish that the ``directional oracle complexities'' of {\lrgd} for strongly convex and non-convex objective functions are $\ccalO(r \log(1/\epsilon) + rp)$ and $\ccalO(r/\epsilon^2 + rp)$, respectively. When $r \ll p$, these complexities are smaller than the known complexities of $\ccalO(p \log(1/\epsilon))$ and $\ccalO(p/\epsilon^2)$ of {\gd} in the strongly convex and non-convex settings, respectively. Thus, {\lrgd} significantly reduces the computational cost of gradient-based methods for sufficiently low-rank functions. In the course of our analysis, we also formally define and characterize the classes of exact and approximately low-rank functions.

\end{abstract}

\section{Introduction}

First order optimization methods such as Gradient Descent ({\gd}) and its variants have become the cornerstone of training modern machine learning models. Hence, reducing the running times of first order methods has been an important problem in the optimization literature, cf. \cite{Bubeck2015,BottouCurtisNocedal2018}. The running times of {\gd} methods is known to grow linearly with the dimension of the model parameters, which can be very large, e.g., in deep neural networks. However, it has recently been observed that many empirical risk minimization problems have objective functions (i.e., real-valued losses) with \emph{low-rank structure} in their gradients \cite{gur2018gradient,Sagunetal2018,Papyan2019,cui2020active,gooneratne2020low,Wuetal2021,JadbabaieMakurShah2022}. In what follows, we will refer to such objects as \textit{low-rank functions}. 

Roughly speaking, a low-rank function is a differentiable real-valued function whose gradients live \emph{close} to a low-dimensional subspace. Such low-rank structure has been exploited to theoretically improve the running times of federated optimization algorithms in \cite{JadbabaieMakurShah2022}. Yet, canonical {\gd} methods do not exploit this additional structure. Hence, the goal of this work is to address the following question:
\begin{tcolorbox}
\begin{center}
\textit{Can we leverage low-rank structure to design optimization algorithms with running times that have better dependence on the dimension?}
\end{center}
\end{tcolorbox}

We consider solving the optimization problem $\min_{\theta \in \reals^p} f(\theta)$, with objective function $f:\Rbb^p \rightarrow \Rbb$, using gradient-based methods. In particular, we characterize the computational cost of an iterative routine to solve this minimization problem, namely \emph{oracle complexity}, as the number of evaluations of \emph{directional derivatives} of $f$ to achieve a certain accuracy level $\epsilon >0$. Under this definition, the vanilla {\gd} algorithm requires an oracle complexity of $\ccalO(p \log(1/\epsilon))$ to find an $\epsilon$-minimizer of a strongly convex objective $f$, which grows linearly with parameter dimension $p$. Note that each iteration of {\gd} costs $p$ directional derivatives as it requires a full gradient computation. 

To ameliorate the oracle complexity of such methods, we propose to leverage the existing low-rank structure in the objective. As briefly pointed out above, low-rank functions demonstrate significant variation in only a few directions of the parameter space, e.g., in $r$ directions where $r \ll p$. As a result, its gradient vectors live entirely (or approximately) in a low-dimensional subspace.
To minimize such low-rank objective functions, we restrict the function to the low-rank subspace defined by the significant directions -- the \emph{active subspace} -- and perform descent iterations only along these directions. This is the main idea behind our proposed method \emph{Low-Rank Gradient Descent} ({\lrgd}). More precisely, {\lrgd} first identifies a fairly accurate proxy for the active subspace. Then, in each descent iteration, it approximates the true gradient vector on the current iterate by computing only $r$ directional derivatives of the objective along the active subspace. This approximate gradient is then used to update the model parameters. 
Note that each iteration of {\lrgd} costs only $r$ directional derivative computations. Intuitively, this is far less than canonical methods such as {\gd}, which has iteration cost growing linearly with $p$, as noted earlier.

Low-rank structures have been observed in several contexts which further highlights the potential utility of the proposed {\lrgd} method. We briefly describe some such motivating examples below.

\textbf{Motivation 1: Low-rank Hessians.} Low-rank structures have been found in large-scale deep learning scenarios irrespective of the architecture, training methods, and tasks \cite{gur2018gradient,Sagunetal2018,Papyan2019,Wuetal2021,singh2021analytic}, where the Hessians exhibit a sharp decay in their eigenvalues. For instance, in classification tasks with $k$ classes, during the course of training a deep neural network model, the gradient lives in the subspace spanned by the $k$ eigenvectors of the Hessian matrix with largest eigenvalues \cite{gur2018gradient}. As another example, \cite{cui2020active} empirically demonstrates that some layers of a deep neural network model (VGG-19) trained on the CIFAR-10 dataset shows \emph{exponential} decay in the eigenvalues of the gradient covariance matrix. These practical scenarios further suggest that the {\lrgd} method may be able to leverage such low-rank structures to mitigate training computation cost. This connection is further detailed in Appendix \ref{ap: discussion_low_rank}. Furthermore, a simple simulation based on the the MNIST database \cite{LeCunCortesBurgesMNIST} that illustrates low-rank structure in gradients of neural networks is provided in Appendix \ref{ap: Low-rank objective functions of neural networks}.

\textbf{Motivation 2: Relation to line search.} Line search is an iterative strategy to find the optimum stepsize in {\gd}-type methods \cite{boyd2004convex}. More precisely, given a current position $\theta$, the line search algorithm minimizes the objective $f$ restricted to the rank-$1$ subspace (i.e., a line) passing through $\theta$ in the direction of $\gr f(\theta)$, that is, $\theta \gets \argmin_{\theta'\in \{\theta + \alpha \nabla f(\theta) : \alpha \in \Rbb\}} f(\theta')$. Now consider a similar search problem where the objective $f$ is to be minimized  restricted to a rank-$r$ subspace rather than a line. 
We refer to this subspace search method as \emph{iterated {\lrgd}} (see Algorithm \ref{alg: subspace search}). The proposed {\lrgd} method can be viewed as solving the intermediate minimization problems in such subspace search.

\textbf{Motivation 3: Ridge functions.} Ridge functions are generally defined as functions that only vary on a given low-dimensional subspace of the ambient space \cite{logan1975optimal}. There are many standard examples of ridge function losses in machine learning, e.g., least-squares regression, logistic regression, one hidden layer neural networks, etc. \cite{ismailov2020notes, pinkus1999approximation}. Moreover, they have been exploited in the development of projection pursuit methods in statistics, cf. \cite{donoho1989projection} and the references therein. In the next section we will show that the $\lrgd$ method is particularly well-suited for optimization of such functions. 

\textbf{Motivation 4: Constrained optimization.} The general idea of solving an optimization problem on a smaller dimensional subspace is connected to \emph{constrained optimization}. In fact, one of the primary steps of $\lrgd$ can be perceived as \emph{learning} (through sampling) ``hidden'' linear constraints under which it is efficient to solve an a priori unconstrained optimization problem. Classically, when such constraints are known, e.g., when optimizing $f(\theta)$ under a constraint of the form $A \theta = b$, a change-of-variables allows us to reduce the dimension of the optimization problem \cite{boyd2004convex}.

\textbf{Main contributions.} We next list our main contributions: 
\begin{table} \label{tab: OC}
\begin{center}
\begin{tabulary}{11.5cm}{RCCC}
\cmidrule[\heavyrulewidth]{2-3}&
strongly convex & non-convex
\\ \midrule
{\gd} \par {\lrgd} &
 $\ccalO(p \log(1/\epsilon))$ \cite{Nesterov2004} \par \textcolor{blue}{$\ccalO(r \log(1/\epsilon) + rp)$} [Th \ref{thm: exact LR SC}, \ref{thm: apprx LR SC}]  &
 $\ccalO(p/\epsilon^2)$ \cite{Nesterov2004} \par \textcolor{blue}{$\ccalO(r/\epsilon^2 + rp)$} [Th \ref{thm: exact LR NC}, \ref{thm: apprx LR NC}]
\\\bottomrule
\end{tabulary}
\vspace{0.1in}
\caption{Oracle complexities for both exactly and approximately low-rank settings. (The difference between these settings is in constants that are hidden by the $\ccalO$ notation.)}
\end{center}
\end{table}
\begin{enumerate}[label=\arabic*),leftmargin=*,noitemsep,nolistsep]
\item We identify the class of low-rank functions and propose {\lrgd} in Algorithm \ref{alg: LRGD} to mitigate gradient computation cost of {\gd}-type methods to minimize such objectives. 
\item We provide theoretical guarantees for the proposed {\lrgd} method and characterize its oracle complexity in optimizing both exactly and approximately low-rank strongly convex and non-convex functions in Theorems \ref{thm: exact LR SC}, \ref{thm: apprx LR SC}, \ref{thm: exact LR NC}, and \ref{thm: apprx LR NC}. As demonstrated in Table \ref{tab: OC}, for low-rank objectives (with sufficient approximation accuracy), {\lrgd} is able to reduce the dependency of the dominant term in {\gd} from $p$ to $r$. In particular, compared to {\gd}, {\lrgd} slashes the oracle complexity of finding an $\epsilon$-minimizer from $\ccalO(p \log(1/\epsilon))$ to $\ccalO(r \log(1/\epsilon) + rp)$ and from $\ccalO(p/\epsilon^2)$ to $\ccalO(r/\epsilon^2 + rp)$, respectively, for strongly convex and non-convex objectives. 

\item We derive several auxiliary results characterizing exactly and approximately low-rank functions, e.g., Propositions \ref{prop: exact LR conditions}, \ref{prop: gradients_span_directions}, and \ref{prop: algebra of LR}. 

\item We propose several algorithms that can optimize general (possibly high-rank) functions using {\lrgd} as a building block, e.g., \emph{iterated {\lrgd}} (Algorithm \ref{alg: subspace search}) and \emph{adaptive {\lrgd}} (Algorithm \ref{alg: adaptive LRGD}). While $\lrgd$ is provably efficient on (approximately) low-rank functions, its variants are \emph{applicable to broader classes of functions that may not possess low-rank structure} on the full space.

\item We evaluate the performance of {\lrgd} on different objectives and demonstrate its usefulness in restricted but insightful setups. 
\end{enumerate}

\textbf{Related work.} Low-rank structures have been exploited extensively in various disciplines. For example, in machine learning, matrix estimation (or completion) methods typically rely on low-rank assumptions to recover missing entries \cite{CandesPlan2010}. In classical statistics, projection pursuit methods rely on approximating functions using ridge functions \cite{logan1975optimal,donoho1989projection}, which are precisely low-rank functions (see Proposition \ref{prop: exact LR conditions}). In a similar vein, in scientific computing, approximation methods have been developed to identify influential input directions of a function for uncertainty quantification \cite{constantine2014active}.

In contrast to these settings, we seek to exploit low-rank structure in optimization algorithms. Recently, the authors of \cite{JadbabaieMakurShah2021} developed a local polynomial interpolation based {\gd} algorithm for empirical risk minimization that learns gradients at every iteration using smoothness of loss functions in data. The work in \cite{JadbabaieMakurShah2022} extended these developments to a federated learning context (cf. \cite{McMahanetal2017}, \cite{Reisizadehetal2020} and the references therein), where low-rank matrix estimation ideas were used to exploit such smoothness of loss functions. Following this line of reasoning, this paper utilizes the structure of low-rank objective functions to improve iterative gradient-based algorithms in the canonical optimization setting. (In a very different direction, low-rank structure has also been used to solve semidefinite programs in \cite{chen2014efficient}.)
 
Several works have recently highlighted different forms of low-rank structure in large-scale training problems. For example, deep neural networks seem to have loss functions with low-rank Hessians \cite{gur2018gradient,Sagunetal2018,Papyan2019,Wuetal2021}, which partly motivates our work here (also see Appendix \ref{ap: Low-rank objective functions of neural networks}). Moreover, deep neural networks have been shown to exhibit low-rank weight matrices \cite{LeJegelka2022,galanti2022sgd} and neural collapse \cite{PapyanaHanDonoho2020,Fangetal2021,HuiBelkinNakkiran2022}, and low-rank approximations of gradient weight matrices have been used for their training \cite{gooneratne2020low}. 

Our work falls within the broader effort of speeding up first order optimization algorithms, which has been widely studied in the literature. In this literature, the running time is measured using first order oracle complexity (i.e., the number of full gradient evaluations until convergence) \cite{NemirovskiiYudin1983,Nesterov2004}. The first order oracle complexity of {\gd} for, e.g., strongly convex functions, is analyzed in the standard text \cite{Nesterov2004}. Similar analyses for stochastic versions of {\gd} that are popular in large-scale empirical risk minimization problems, such as (mini-batch) stochastic {\gd}, can be found in \cite{Nemirovskietal2009,Dekeletal2012}. There are several standard approaches to theoretically or practically improving the running times of these basic algorithms, e.g., momentum \cite{polyak1964some}, acceleration \cite{Nesterov1983}, variance reduction \cite{LeRouxSchmidtBach2012,SchmidtLeRouxBach2017,ShalevShwartzZhang2013,JohnsonZhang2013}, and adaptive learning rates \cite{duchi2011adaptive,tieleman2012lecture,KingmaBa2015}. More related to our problem, random coordinate descent-type methods such as stochastic subspace descent \cite{kozak2019stochastic} have been studied. In addition, various other results pertaining to {\gd} with inexact oracles (see \cite{FriedlanderSchmidt2012} and the references therein), fundamental lower bounds on oracle complexity \cite{Agarwaletal2009,Carmonetal2019}, etc. have also been established in the literature. We refer readers to the surveys in \cite{Bubeck2015,BottouCurtisNocedal2018} for details and related references.

The contributions in this paper are complementary to the aforementioned approaches to speed up first order optimization methods, which do not use low-rank structure in objective functions. Indeed, ideas like acceleration, variance reduction, and adaptive learning rates could potentially be used in conjunction with our proposed {\lrgd} algorithm, although we leave such developments for future work. Furthermore, we only consider {\gd}-like methods rather than more prevalent stochastic {\gd} methods in this work, because our current focus is to show simple theoretical improvements in running time by exploiting low-rank structure. Hence, we also leave the development of stochastic optimization algorithms that exploit low-rank structure for future work.

\textbf{Outline.} We briefly outline the remainder of the paper. We state our assumptions, the computational model, and definitions of low-rank functions in Section \ref{sec: setup}. We describe the {\lrgd} algorithm and its theoretical guarantees in Section \ref{sec: results}. Finally, illustrative simulations are given in Section \ref{sec: sims}. \label{sec: intro}

\section{Formal setup and low-rank structure} \label{sec: setup}

\textbf{Notation.} We refer readers to Appendix \ref{ap: notation} for a list of notation used in this paper. 

\subsection{Preliminaries}

We consider a real-valued differentiable function $f :\Theta \rightarrow \Rbb$, where $\Theta = \Rbb^p$ for some positive integer $p$, and denote its gradient $\nabla f : \Theta \rightarrow \Rbb^p$. We assume that $f$ is $L$-smooth, i.e, that its gradient is $L$-Lipschitz continuous. 

\begin{assumption}[$L$-smoothness] \label{assumption: smooth}
The function $f$ is $L$-smooth for a given parameter $L>0$ if for any $\theta, \theta' \in \Theta$, $f(\theta') 
\leq
f(\theta) + \langle \gr f(\theta), \theta' - \theta \rangle + \frac{L}{2} \Vert \theta' - \theta \Vert^2$.
\end{assumption}

\textbf{Non-convex setting.} Under no further assumption on the function $f$, the general goal we will pursue is to find an $\epsilon$-stationary point, that is, for a given accuracy $\epsilon$, we aim to find $\theta \in \Theta$ such that,
\begin{align} \label{eq: e-stationary}
    \Vert \nabla f(\theta)\Vert \leq \epsilon.
\end{align}
Note that a solution to \eqref{eq: e-stationary} exists as soon as $f$ is lower-bounded by some $f^* = \inf_{\theta\in \Theta} f(\theta) \in \Rbb$ which is assumed throughout this paper. 

\textbf{Strongly convex setting.} We also study a similar problem under the additional assumption that $f$ is strongly convex. 
\begin{assumption}[Strong convexity] \label{assumption: SC} The function $f$ is $\mu$-strongly convex for a given parameter $\mu>0$ if for any $\theta, \theta' \in \Theta$,
$f(\theta') 
\geq
f(\theta) + \langle \gr f(\theta), \theta' - \theta \rangle + \frac{\mu}{2} \Vert \theta' - \theta \Vert^2$.
\end{assumption}
Under this particular setting, a solution of \eqref{eq: e-stationary} can be interpreted as an approximate minimizer of $f$ through the \emph{Polyak-Łojasiewicz inequality} (see, e.g., \cite{boyd2004convex}):
$f(\theta)-f^*\leq \frac{1}{2\mu}\|\nabla f(\theta)\|^2$.
The goal we will pursue will be to find an $\epsilon$-minimizer for $f$, which is defined as $\theta \in \Theta$ satisfying
\begin{align} \label{eq: e-minimizer}
    f(\theta)- f^* \leq \epsilon,
\end{align}
where $\epsilon > 0$ is the predefined accuracy. Note that it suffices to find an $\epsilon'$-stationary point of $f$ with $\epsilon' = \sqrt{2\mu\epsilon}$ to get an $\epsilon$-minimizer of $f$, which explains the relation between both settings. 

In this paper, we focus on gradient descent methods and zoom-in on their computational complexity which we concretely measure through the following computation model.

\textbf{Computation model.} To characterize the running time of a {\gd}-type algorithm for solving the problem \eqref{eq: e-stationary} and \eqref{eq: e-minimizer}, we employ a variant of the well-established notion of \emph{oracle complexity} \cite{NemirovskiiYudin1983,Nesterov2004}. In particular, we tailor this notion to count the number of calls to \emph{directional} derivatives of the objective where the directional derivative of $f$ along a unit-norm vector $u \in \reals^p$ is a scalar defined as, 
\begin{align}
\partial_{u} f(\theta) 
\coloneqq
\lim_{t\rightarrow 0}\frac{f(\theta+t u)-f(\theta)}{t}
=
\langle \nabla f(\theta), u \rangle.
\end{align}
This computation model -- which differs from most of the literature on first order methods that typically assumes the existence of a full-gradient oracle $\nabla f(\theta)$ -- will allow us to showcase the utility of a low-rank optimization method. 
\begin{definition}[Oracle complexity]
Given an algorithm {\normalfont{\texttt{ALG}}}, for a predefined accuracy $\epsilon >0 $ and function class $\Fcal$, we denote by $\ccalC_{{\normalfont{\texttt{ALG}}}}$ the maximum number of oracle calls to directional derivatives required by {\normalfont{\texttt{ALG}}} to reach an $\epsilon$-solution to \eqref{eq: e-minimizer} or \eqref{eq: e-stationary} for any function in $\Fcal$. 
\end{definition}

Note that computing a full-length gradient vector of dimension $p$ requires $p$ calls to the directional gradient oracle. As a consequence, for the class of smooth and strongly convex functions, the oracle complexity of vanilla Gradient Descent algorithm to find an $\epsilon$-minimizer of $f$ is $\ccalC_{{\gd}} = \ccalO(p \log(1/\epsilon))$.

In this paper, we target the (directional) oracle complexity of gradient-based algorithms such as {\gd} and propose a computation-efficient algorithm, namely \emph{Low-Rank Gradient Descent} ({\lrgd}). The main idea of {\lrgd} is to leverage the potential low-rank structure of the objective in order to slash the oracle complexity during the training. To do so, {\lrgd} first identifies a low-rank subspace $H$ that (approximately) contains the gradients of the objective $f$ (also known as \emph{active subspace}). Let $\{u_1,\cdots,u_r\}$ denote a basis for $H$. In each iteration with current parameter model $\theta$, {\lrgd} computes $r$ directional derivatives of $f$, i.e. $\partial f_{u_1}(\theta), \cdots, \partial f_{u_r}(\theta)$, and uses
\begin{align}
\grhat f(\theta)
=
\sum_{i\in [r]}\partial_{u_i}f(\theta) u_i,
\end{align}
as an approximation to the true gradient $\gr f(\theta)$ and updates $\theta \gets \theta - \alpha \grhat f(\theta)$ with a proper stepsize $\alpha$. We provide the description of {\lrgd} and its guarantees for different function classes in Section \ref{sec: results}.

\subsection{Exactly low-rank functions}

In this paragraph, we formally define the notion of \textit{low-rank function} (Definition \ref{definition: excat LR}). We then immediately observe that such functions have already widely been studied in the literature under different equivalent forms (Proposition \ref{prop: exact LR conditions}). We make the simple observation that the class of such functions is fairly restricted, e.g., it contains no strongly convex function, thereby requiring the more general notion of approximately low-rank functions that will come later.  
\begin{definition}[Rank of function] \label{definition: excat LR}
The function $f$ has rank $r$ if the minimal subspace $H\subseteq \Rbb^p$ satisfying $\nabla f(\theta) \in H$ for any $\theta \in \Theta$ has dimension $r$. 
\end{definition}

\begin{proposition}[Characterizations of low-rank functions] \label{prop: exact LR conditions} The following are equivalent:
\begin{enumerate}[label=\arabic*),ref=(\arabic*),leftmargin=*,noitemsep,nolistsep]
\item There exists a subspace $H$ of dimension $r$ such that $\nabla f(\theta)\in H$ for all $\theta\in \Theta$.  
\label{statement1}
\item There exists a subspace $H$ of dimension $r$ such that $f(\theta) = f(\theta+\theta')$ for all $\theta\in \Theta$ and $\theta' \in H^{\perp}$.\label{statement2} \item There exists a matrix $A\in \Rbb^{r \times p}$ and a map $\sigma:\Rbb^r \rightarrow \Rbb$ such that $f(\theta) = \sigma(A\theta)$ for all $\theta\in \Theta$. \label{statement3}
\end{enumerate}
\end{proposition}
We defer the proof to Appendix \ref{proof: prop exact LR}. The proposition above relates our notion of low-rank functions 1) to other equivalent forms. The property 2) is ordinarily referred to as $H^{\perp}-$invariance (see, e.g. \cite{constantine2014active}) whereas property 3) is typically used to define the notion of ridge functions \cite{logan1975optimal} or multi-ridge functions \cite{tyagi2014learning}. In what follows, we will equivalently use the terminology of (exactly) low-rank function and ridge function. 

Note that for our purposes, the class of ridge functions has some limitations. In particular, any $\mu-$strongly convex function is of (full) rank $p$ because it satisfies:
\begin{equation}
    f(\theta^*+\theta') \geq f(\theta^*)+\frac{\mu}{2}\|\theta'\|^2
    >
    f(\theta^*), \quad \forall \theta'\in \Rbb^p, \theta' \neq 0, 
\end{equation}
which together with 2) in Proposition \ref{prop: exact LR conditions} yields that $H$ is of full dimension $r=p$. On the other hand, we expand on some properties of low-rank functions in Appendix \ref{Algebra of Low Rank Functions}. Furthermore, we illustrate via an example in Appendix \ref{Example: Transformation to a Low Rank Function} that optimization of a high-rank function can sometimes be equivalently represented as optimization of a low-rank function using an appropriately chosen non-linear transformation.

\subsection{Approximately low-rank functions}\label{sec: approx_lr}

The discussion above calls for a relaxation of the notion of low rank functions. In this section, we provide a definition for \emph{approximately low-rank} functions where we no longer require that the gradients all belong to a given low-dimensional subspace $H$ but instead are inside a cone supported by this subspace, this is made explicit in Definition \ref{definition: apprx LR}. In the rest of this section, more precisely through Proposition \ref{prop: apprx LR} and Proposition \ref{prop: gradients_span_directions} we give more details about the desired properties of the approximation.

\begin{definition}[Approximate rank of function] \label{definition: apprx LR}
The function $f$ is $(\eta, \epsilon)$-approximately rank-$r$, if there exist $\eta \in [0,1)$, $\epsilon > 0$ and subspace $H\subseteq \Rbb^p$ of dimension $r \leq p$, such that for any $\theta \in \Theta$,
\begin{equation*}
    \Vert\nabla f(\theta) - \Pi_H(\nabla f(\theta)) \Vert \leq \eta\Vert\nabla f(\theta)\Vert+\epsilon,
\end{equation*}
where $\Pi_H$ denotes the orthogonal projection operator onto $H$.
\end{definition}
Note that this condition has two error terms represented by two constants $\epsilon$ (additive) and $\eta$ (multiplicative). It therefore considerably relaxes the exact low rank assumption that would require $\forall \theta \in \Rbb^p : \Vert\nabla f(\theta) - \Pi_H(\nabla f(\theta)) \Vert = 0$. Both of these constants are required for the generality of the definition of the approximately low-rank condition. In particular, the role of the additive term is emphasized in the following Proposition \ref{prop: apprx LR} whereas the role of the multiplicative term is emphasized in Proposition \ref{prop: gradients_span_directions} for the particular case of strongly convex functions. 

\begin{proposition}[Non-singular Hessian and approximate low-rankness] \label{prop: apprx LR}
If $f$ attains its minimum at $\theta^*\in \Theta$ where $\nabla^2f(\theta^*)$ is invertible and $f$ is approximately low-rank with parameters $(\eta,\epsilon)$ where $\eta<1$ then necessarily $\epsilon > 0$.
\end{proposition}
The proof in Appendix \ref{proof: prop apprx LR} is a direct application of the inverse function theorem to $\nabla f$ at $\theta^*$. This result shows that in most settings the approximate rank requires an affine approximation error (i.e., $\epsilon > 0$ is necessary). Note that this does not produce a uniform bound on $\epsilon$ that would hold for any function $f$; such a bound would only come under additional assumptions on the third derivatives of $f$. 

\begin{proposition}[Strong convexity and approximate low-rankness]\label{prop: gradients_span_directions}
If $f$ is $\mu$-strongly convex, for any vector $u$ of unit norm, and any $\Delta>0$ there exists some $\theta\in \Rbb^p$ such that,
\begin{equation*}
    \frac{\nabla f(\theta)}{\|\nabla f(\theta)\|} = u,\quad \text{and}\quad \Vert\nabla f(\theta)\Vert\geq \mu\Delta, \quad \text{where} \quad \|\theta-\theta^*\|<\Delta.
\end{equation*}
As a consequence, if $f$ is $(\eta, \epsilon)$-approximate rank $r<p$ on the ball of radius $\Delta$ around $\theta^*$, i.e., $\Bcal(\theta^*,\Delta)$, then $\epsilon >\mu\Delta(1-\eta)$ or alternatively $\eta>1 - \frac{\epsilon}{\mu \Delta}$.
\end{proposition}
The proof in Appendix \ref{proof: prop gradients_span_directions} shows that the gradient of strongly convex functions typically span all directions. As a consequence, this result shows that there is a trade-off between the strong convexity constant $\mu$ and the approximately low-rank constants $(\eta, \epsilon)$. For a given strongly convex function $f$ it is however still very reasonable to assume that it is approximately low-rank not on the full space $\Theta = \Rbb^p$ but on bounded portions of the optimization space of the form $\Bcal(\theta^*,\Delta)$ with limited $\Delta$. This will turn out to be sufficient in most practical cases.

\section{Algorithms and theoretical guarantees} \label{sec: results}

In this section, we state the precise description of the proposed {\lrgd} algorithm, we provide its computational complexity guarantees and conclude with a discussion on practical considerations.

\subsection{{\lrgd} algorithm}

The proposed algorithm, {\lrgd}, is an iterative gradient-based method. It starts by identifying a subspace $H$ of rank $r \leq p$ that is a good candidate to match our definition of approximately low-rank function for the objective $f$, i.e., Definition \ref{definition: apprx LR}. More precisely, we pick $r$ random  models $\theta^1, \cdots, \theta^r$ and construct the matrix $G=[g_1/\Vert g_1 \Vert, \cdots, g_r/\Vert g_r \Vert]$ where we denote $g_j \coloneqq \gr f(\theta^j)$ for $j \in [r]$. Then, assuming that $G$ is full-rank, the active subspace $H$ will be the space induced by the $r$ dominant left singular vectors of $G$. In other words, if we denote $G = U \Sigma V^\top$ the singular value decomposition (SVD) for $G$, we pick $H =\mathrm{span}(U)= \mathrm{span}(u_1,\cdots,u_r)$. Having set the active subspace $H$, {\lrgd} updates the model parameters $\theta_t$ in each iteration $t$ by $\theta_{t+1} = \theta_t - \alpha \grhat f(\theta_t)$ for a proper stepsize $\alpha$.
Note that $\grhat f(\theta_t)$ here denotes the projection of the true gradient $\gr f(\theta_t)$ on the active subspace $H$ which {\lrgd} uses as a low-rank proxy to  $\gr f(\theta_t)$. Computing each approximate gradient $\grhat f$ requires only $r$ (and not $p$) calls to the directional gradient oracle as we have $\grhat f(\theta_t) = \sum_{j=1}^r \partial_{u_j}f(\theta_t)u_j$. Algorithm \ref{alg: LRGD} provides the details of {\lrgd}.

\begin{algorithm}[t]
\textbf{Require:} rank $r \leq p$, stepsize $\alpha$

pick $r$ points $\theta^1, \cdots, \theta^r$ and evaluate gradients $g_j = \gr f(\theta^j)$ for $j \in [r]$

set $G=[g_1/\Vert g_1 \Vert, \cdots, g_r/\Vert g_r \Vert]$

compute SVD for $G$: $G = U \Sigma V^\top$ with $U=[u_1, \cdots, u_r]$ \hfill{\small{{\textit{\quad (Gram-Schmidt or QR is also possible)}}}}

set $H = \mathrm{span}(u_1,\cdots,u_r)$

initialize $\theta_0$

\For{$t=0,1,2,\cdots$}{

compute gradient approximation $\grhat f(\theta_t) = \Pi_H(\gr f(\theta_t)) = \sum_{j=1}^r \partial_{u_j}f(\theta_t)u_j$

update $\theta_{t+1} = \theta_t - \alpha \grhat f(\theta_t)$
}
\caption{Low-Rank Gradient Descent (\lrgd)}\label{alg: LRGD}
\end{algorithm}

\subsection{Oracle complexity analysis for strongly convex setting}

Next, we characterize the oracle complexity of the proposed {\lrgd} method for strongly convex objectives and discuss its improvement over other benchmarks. Let us start from a fairly simple case. As elaborated in Section \ref{sec: setup} and in Proposition \ref{prop: exact LR conditions}, strongly convex functions cannot be exactly low-rank. However, an exactly low-rank function may be strongly convex when restricted to its active subspace. Next theorem characterizes the oracle complexity for {\lrgd} in such scenario.


\begin{theorem}[Oracle complexity in exactly low-rank and strongly convex case] \label{thm: exact LR SC}

Let the objective function $f$ be $L$-smooth (Assumption \ref{assumption: smooth}) and exactly rank-$r$ according to Definition \ref{definition: excat LR} where $\gr f(\theta) \in H$, $\forall \theta$. Moreover, assume that $f$ restricted to $H$ is $\mu$-strongly convex (Assumption \ref{assumption: smooth}) with condition number $\kappa \coloneqq L / \mu$. Then, the proposed {\normalfont{\lrgd}} method in Algorithm \ref{alg: LRGD} reaches an $\epsilon$-minimizer of $f$ with oracle complexity 
\begin{align}
\ccalC_{{\normalfont{\lrgd}}} = \kappa r \log(\Delta_0/\epsilon) + p r,
\end{align}
where $\Delta_0 = f(\theta_0) - f^*$ is the suboptimality of the initialization $\theta_0$.
\end{theorem}
We defer the proof to Appendix \ref{sec: exact LR SC proof}. Next, we provide oracle complexity guarantees for approximately low-rank and strongly convex functions and discuss the benefit of {\lrgd} to other benchmarks.

\begin{theorem}[Oracle complexity in approximately low-rank and strongly convex case] \label{thm: apprx LR SC}
Let the objective function $f$ be $L$-smooth and $\mu$-strongly convex with condition number $\kappa \coloneqq L / \mu$, i.e., Assumptions \ref{assumption: smooth} and \ref{assumption: SC} hold. We also assume that $f$ is $(\eta,\sqrt{\mu \epsilon/5})$-approximately rank-$r$ according to Definition \ref{definition: apprx LR} and the following condition holds 
\begin{align}
    \eta \left( 1 + \frac{2r}{\sigma_r} \right)
    +
    \frac{2r}{\sigma_r \epsilon'} \sqrt{\frac{\mu \epsilon}{5}}
    \leq \frac{1}{\sqrt{10}},
\end{align}
where $\sigma_r$ is the smallest singular value of the matrix $G = [g_1/{\Vert g_1\Vert},\cdots,g_r/{\Vert g_r\Vert}]$ with $g_j \coloneqq \gr f(\theta^j)$ and $\Vert g_j \Vert > \epsilon'$ for all $j \in [r]$. Then, the proposed {\normalfont{\lrgd}} method in Algorithm \ref{alg: LRGD} reaches an $\epsilon$-minimizer of $f$ with oracle complexity 
\begin{align}
\ccalC_{{\normalfont{\lrgd}}}
=
16 \kappa r \log(2 \Delta_0 / \epsilon) + pr.
\end{align}
\end{theorem}
We defer the proof to Appendix \ref{sec: apprx LR SC proof}. Theorems \ref{thm: exact LR SC} and \ref{thm: apprx LR SC} suggest that in the strongly convex setting and for approximately low-rank functions with sufficiently small parameters, {\lrgd} is able to reduce the oracle complexity of {\gd} which is $\ccalC_{{\normalfont{\gd}}} = \ccalO(\kappa p \log(1 / \epsilon))$ to $\ccalC_{{\normalfont{\lrgd}}} = \ccalO(\kappa r \log(1 / \epsilon) + pr)$. This gain is particularly significant as the term depending on the accuracy $\epsilon$ does not scale with the model parameters dimension $p$, rather scales with the rank $r$ which may be much smaller than $p$.

\subsection{Oracle complexity analysis for non-convex setting}

Next, we turn our focus to non-convex and smooth objectives and discuss the benefits of \texttt{LRGD} on the overall coracle complexity for such functions. Let us begin with exactly low-rank objectives.

 \begin{theorem}[Oracle complexity in exactly low-rank and non-convex case] \label{thm: exact LR NC}
Let the objective function $f$ be $L$-smooth, i.e., Assumption \ref{assumption: smooth} holds. Moreover, assume that the gradient of $f$ is exactly rank-$r$ according to Definition \ref{definition: excat LR}. Then, the proposed {\normalfont{\lrgd}} method in Algorithm \ref{alg: LRGD} reaches an $\epsilon$-stationary point of $f$ with oracle complexity 
\begin{align}
\ccalC_{{\normalfont{\lrgd}}} = \frac{2 r L \Delta_0}{\epsilon^2} + p r.
\end{align}
\end{theorem}
We defer the proof to Appendix \ref{sec: exact LR NC proof}. Next, we state the oracle complexity result for approximately low-rank objectives.

\begin{theorem}[Oracle complexity in approximately low-rank and non-convex case] \label{thm: apprx LR NC}
Let the objective function $f$ be $L$-smooth, i.e., Assumption \ref{assumption: smooth} holds. Moreover, assume that $f$ is $(\eta,\epsilon/3)$-approximately rank-$r$ according to Definition \ref{definition: apprx LR}. Moreover, we assume that the following condition holds
\begin{align}
    \eta \left( 1 + \frac{2r}{\sigma_r} \right)
    +
    \frac{2r \epsilon}{3 \sigma_r \epsilon'} 
    \leq \frac{1}{\sqrt{10}},
\end{align}
where $\sigma_r$ is the smallest singular value of the matrix $G = [g_1/{\Vert g_1\Vert},\cdots,g_r/{\Vert g_r\Vert}]$ with $g_j \coloneqq \gr f(\theta^j)$ and $\Vert g_j \Vert > \epsilon'$ for all $j \in [r]$. Then,  for stepsize $\alpha = \frac{1}{8L}$, {\normalfont{\texttt{LRGD}}} in Algorithm \ref{alg: LRGD} reaches an $\epsilon$-stationary point with oracle complexity 
\begin{align}
    \ccalC_{{\normalfont{\texttt{LRGD}}}} = 
    \frac{72 r L \Delta_0}{\epsilon^2} + p r.
\end{align}
\end{theorem}
We defer the proof to Appendix \ref{sec: apprx LR NC proof}. Theorems \ref{thm: exact LR NC} and \ref{thm: apprx LR NC} further highlight the advantage of {\lrgd} in the non-convex scenario. More precisely, it is well-understood that {\gd} is able to find an $\epsilon$-stationary point of a smooth and non-convex function in $\ccalO(1/\epsilon^2)$ iterations. Since each iteration costs $p$ directional derivatives, the total oracle complexity scales as $\ccalO(p/\epsilon^2)$ where $p$ denotes the dimension of model parameters. This can be costly specifically for deep neural network models with up to millions of parameters.  On the other hand, for approximately rank-$r$ objectives with sufficiently small parameters (see Definition \ref{definition: apprx LR}), {\lrgd} is able to slash the oracle complexity to $\ccalO(r/\epsilon^2 + pr)$. This is particularly significant for low-rank functions with rank $r \ll p$.


Let us point out that although {\lrgd}'s guarantees stated in the previous theorems require the objective to be sufficiently low-rank on the entire model parameters space, the proposed low-rank approach works with less restrictive assumptions on the objective. This is our motivation to build up on {\lrgd} and propose other variants in th following.

\subsection{Variants of {\lrgd}}

Algorithm \ref{alg: LRGD}  ({\lrgd}) can be extended to broader settings, especially in order to make it less sensitive to assumptions on the function $f$, since such assumptions are rarely known a priori before the optimization task is defined. In this section, we present two such extensions: 

\begin{enumerate}[label=\arabic*),leftmargin=*,noitemsep,nolistsep]
    \item Algorithm \ref{alg: subspace search} (in Appendix \ref{ap: adaptive lr}) provides a variant of {\lrgd} where the active subspace is updated as soon as the low-rank descent converges. Note that this guarantees the termination of the optimization algorithm on any function -- and not just on functions satisfying the assumptions of the theoretical results. This algorithm is particularly adapted to situations where the function is \emph{locally} approximately rank $r$ as the active subspace is updated through a local sampling (that is also used to accelerate the descent with $r$ iterations of $\gd$). This algorithm still takes the parameter $r$ as an input. It is used for empirical tests in Section \ref{sec: sims}.
    \item Algorithm \ref{alg: adaptive LRGD} (in Appendix \ref{ap: adaptive lr}) provides a variant of $\lrgd$ that is adapted to situations where the rank $r$ is unknown. In this algorithm, the active subspace is made larger as when the algorithm reaches convergence on the corresponding subspace. As a result the rank $r$ goes from $1$ to $p$. If the function is rank $r^*$ all iterations where the rank $r$ goes from $r^*$ to $p$ will be cost-less. Note that in some circumstances, the arithmetic progression of the rank $\mathrm{update}(r) = r+1$ can be can be advantageously replaced by a geometric progression $\mathrm{update}(r) = 2 r$. The empirical evaluation of this variant is left for future work. 
\end{enumerate}

\section{Numerical simulations} \label{sec: sims}

In this section we attempt to demonstrate the utility of the $\lrgd$ method (specifically, the iterative version in Algorithm \ref{alg: subspace search} in Appendix \ref{ap: adaptive lr}) through numerical simulations. We start in Figure \ref{fig: gd_lrgd_iterations} by a simple two-dimensional example ($p=2$ and $r=1$) to help building intuition about the method. In this experiment we observe that even when the $\lrgd$ method makes a more important number of iterations than the vanilla $\gd$ methods (left), it may be as efficient in terms of oracle calls (right). We then argue with Tables \ref{tab: table half}, \ref{tab: table fourth}, \ref{tab: table tenth} (Tables \ref{tab: table fourth}, \ref{tab: table tenth} are in Appendix \ref{ap: additional_experiment}) that the $\lrgd$ method significantly outperforms the vanilla $\gd$ method in ill-conditioned setups.

\begin{figure}[h!]%
    \centering
    \includegraphics[width=4.5cm]{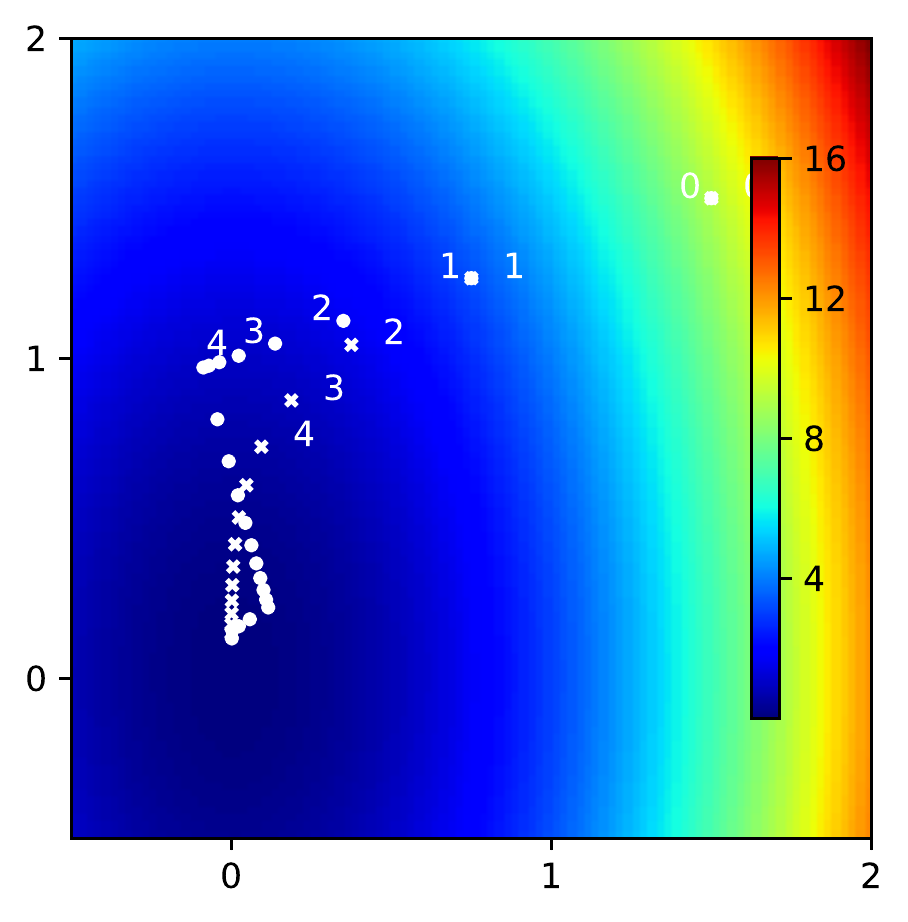}
    \hspace{0.5cm}
    \includegraphics[width=4.5cm]{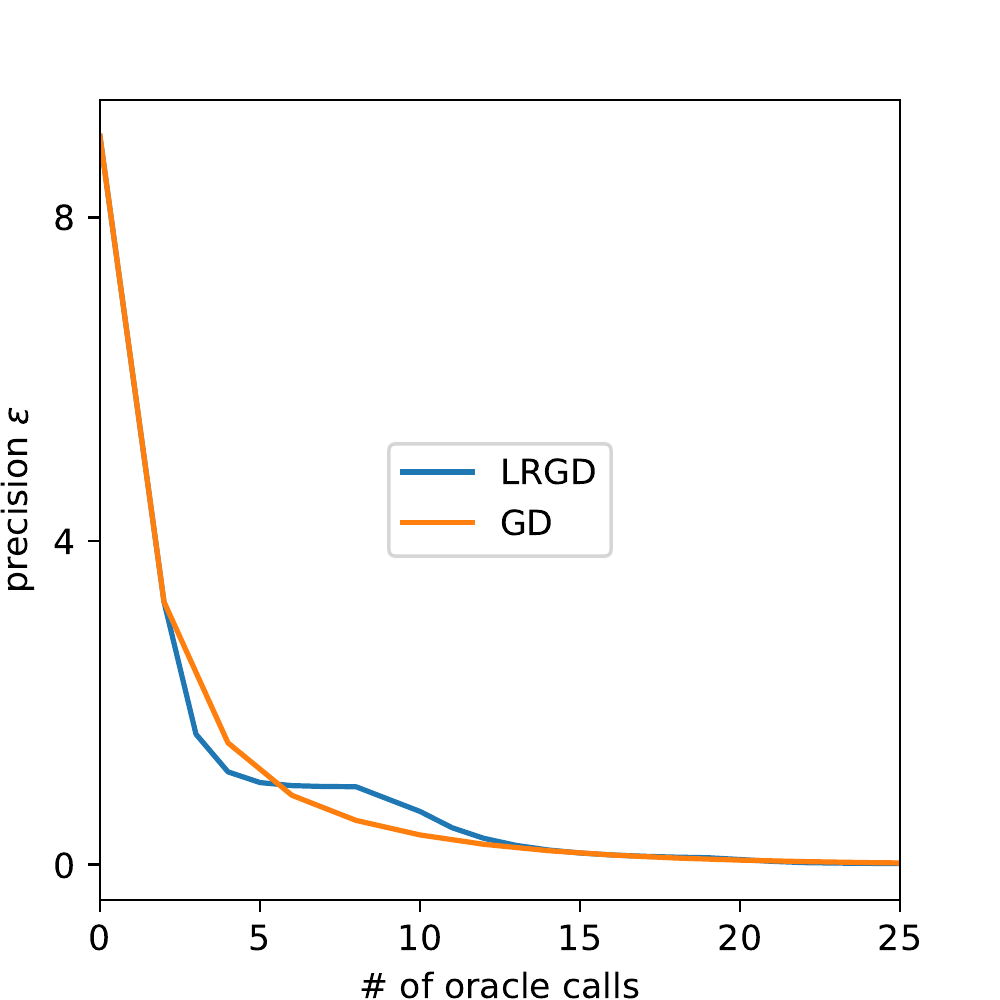}
    \caption{Iterations of $\gd$ (crosses) and $\lrgd$ (dots) on a quadratic function $f(\theta) = \theta^\top H\theta$ with $H = [
    3 \enspace 0; 0 \enspace 1]$
    and $\theta_0 = (1.5 \enspace 1.5)$. Though $\lrgd$ makes more iterations (21 versus 13), it uses slightly less oracle calls (25 versus 26). Experiments ran with $\alpha = 1/15$ and $\epsilon = 0.02$.
    }%
    \label{fig: gd_lrgd_iterations}%
\end{figure}

A close observation of Figure \ref{fig: gd_lrgd_iterations} and of the subsequent tables allows to make the following observations about situations where the $\lrgd$ method outperforms $\gd$.
\begin{enumerate}[label=\arabic*),leftmargin=*,noitemsep,nolistsep]
    \item Even for functions that are not approximately low-rank -- as in the example from Figure \ref{fig: gd_lrgd_iterations} -- using $\lrgd$ instead of $\gd$ does not seem to hurt in terms of number of oracle calls. 
    \item In the results of Table \ref{tab: table half}, $\lrgd$ outperforms $\gd$ by a factor at most  $p/r = 2$ that this factor is nearly attained in some cases (bottom-right of the table). This is in coherence with the theoretical results from Section \ref{sec: results} and the motivation from the introduction stating that the benefits of the $\lrgd$ method are indeed particularly visible when the ratio $p/r$ is higher.
    \item As observed in Figure \ref{fig: gd_lrgd_iterations}, the $\lrgd$ method converges after a series of phases that end with the convergence of the descent on a given subspace. We can distinctly see two phases (iteration 1 to 7 and iteration 7 to 17) and can roughly guess two others (iteration 17 to 19 and iteration 19 to 21). These phases traduce in a series of waterfalls (Figure \ref{fig: gd_lrgd_iterations}, right) that each tend to be sharper than the single waterfall of $\gd$. 
    \item The $\lrgd$ method is particularly efficient when consecutive gradient evaluations are very correlated. This typically because when the step-size $\alpha$ is too small with respect to the local variations of the function. The step-size is determined to satisfy $\alpha<1/L$. It is therefore smaller for ill-conditioned functions -- i.e., examples on the right of Table \ref{tab: table half}.
\end{enumerate}

\begin{table}[h!]
  \centering
  \begin{tabular}{ccccc}
    \toprule
    \multicolumn{5}{c}{$H$}                   \\
    \cmidrule(r){2-5}
    $\theta_0$     
    & 
    $\begin{pmatrix}
    1 & 0\\
    0 & 1
    \end{pmatrix}$
    & 
    $\begin{pmatrix}
    10 & 0\\
    0 & 1
    \end{pmatrix}$
    &
    $\begin{pmatrix}
    100 & 0\\
    0 & 1
    \end{pmatrix}$
    &
    $\begin{pmatrix}
    1000 & 0\\
    0 & 1
    \end{pmatrix}$
    \\
    \midrule
    $\begin{pmatrix} 1.00 & 0.00 \end{pmatrix}$ & 2 (5)  & 2 (5) & 2 (5) & 2 (5)  \\
    $\begin{pmatrix} 0.87 & 0.50 \end{pmatrix}$ & 2 (5) & 22 (20) & 230 (130) & 2302 (1190)  \\
    $\begin{pmatrix} 0.71 & 0.71 \end{pmatrix}$ & 2 (5) & 30 (26) & 300 (166) & 3004 (1536)  \\
    $\begin{pmatrix} 0.50 & 0.87 \end{pmatrix}$ & 2 (5) & 34 (31) & 340 (195) & 3410 (1746)  \\
    $\begin{pmatrix} 0.00 & 1.00 \end{pmatrix}$ & 2 (5) & 36 (22) & 368 (188) & 3688 (1848)  \\
    \bottomrule
  \end{tabular}
  \vspace{2mm}
  \caption{Number of oracle calls before convergence (with $\epsilon = 0.1$) of $\gd$ and $(\lrgd)$ with $r=1$ on quadratic functions, $f(\theta) = \theta^{\top} H \theta$, for different Hessians $H$ with varying condition numbers and different initialization $\theta_0$ on the unit circle. Step size $\alpha$ is set to optimal for $\gd$, i.e.,
  $1/(2L)$ where $L\in \{1, 10, 100, 1000\}$.
  }
  \label{tab: table half}
\end{table}

\textbf{Conclusion.} This work proposes an optimization approach that leverages low-rank structure to reduce the number of directional derivative queries. We believe that other optimization methods could benefit from incorporating such an approach.


\section*{\center Appendices}


\appendix

\section{Notation}\label{ap: notation}

We outline some key notation used in this work. For any vectors $x,y \in \R^p$, we let $\langle x,y\rangle$ denote the Euclidean inner product between $x$ and $y$, and $\|x\|$ denote the Euclidean $\ell^2$-norm. For a given $r>0$, we let $\Bcal(x,r) = \{y\in \Rbb^p : \|x-y\|\leq r\}$ denote the closed ball of radius $r$ with center $x$. For any linear subspace $H \subseteq \R^p$, we use $H^{\perp}$ to represent its orthogonal complement subspace. We define $\Pi_H : \R^p \rightarrow H$ as the orthogonal projection operator onto $H$ such that for all $x \in \Rbb^p$, $ \Pi_H(x) \coloneqq \argmin_{y\in H}\|x-y\|$. Note that for any orthonormal basis $\{u_1, \dots, u_r\}$ of $H$, we have
$\Pi_H(x) = \sum_{i\in [r]}\langle x , u_i \rangle  u_i$. For any matrix $A \in \R^{p \times q}$, we let $A^{\top}$ be the transpose of $A$, $\|A\|_{\mathrm{F}}$ be the Frobenius norm of $A$, $\|A\|_{\op} \coloneqq \max_{x \in \R^q : \|x\|=1}{\|Ax\|}$ be the operator norm of $A$, and $A^{-1}$ be the inverse of $A$ (when it exists and $p = q$). For any collection of vectors $\{x_1,\dots,x_k\} \subseteq \R^p$, we let $\mathrm{span}(x_1,\dots,x_k)$ denote the linear span of the vectors. Moreover, for any matrix $A=[a_1,\dots,a_q] \in \R^{p \times q}$, where $a_1,\dots,a_q \in \R^p$ are the columns of $A$, we let $\mathrm{span}(A)$ be the linear span of the columns of $A$, i.e., the column space of $A$.  

For any sufficiently differentiable real-valued function $f : \Rbb^p \rightarrow \Rbb$, we let $\gr f : \Rbb^p \rightarrow \Rbb^p$ denote its gradient and $\gr^2 f : \Rbb^p \rightarrow \Rbb^{p \times p}$ denote its Hessian. Further, for any unit vector $u \in \Rbb^p$, we let $\partial_u f : \R^p \rightarrow \R$ denote the directional derivative of $f$ in the direction $u$, $\partial_u f(x) \coloneqq \langle\nabla f(x),u \rangle$. 

Throughout this work, $\exp(\cdot)$ and $\log(\cdot)$ denote the natural exponential and logarithm functions (with base $e$). For any positive integer $r$, we let $[r] \coloneqq \{1,\dots,r\}$. Finally, we use standard asymptotic notation, e.g., $o(\cdot)$, $\ccalO(\cdot)$, where the limiting variable should be clear from context.

\section{Variants of low-rank gradient descent algorithm}\label{ap: adaptive lr}

In this appendix, we include the pseudocode for the iterated and adaptive low-rank gradient descent algorithms, which build on {\lrgd} to optimize general functions (that may not be low-rank).

\begin{algorithm}[h!]
\caption{Iterated Low-Rank Gradient Descent}\label{alg: subspace search}

\textbf{Require:} $r \geq 0$, $\epsilon>0$, $\alpha \in \big(0,\frac{1}{L}\big]$, $\mu > 0$, $\theta_0 \in \R^p$

initialize $\theta \gets \theta_0$

\While{$\Vert\nabla f(\theta)\Vert>\sqrt{2\mu\epsilon}$}{
    $\theta_1,...,\theta_r = \gd(\theta,r)$ 
    \hfill{run $r$ iterations of gradient descent}
    
    $u_1,...,u_r = \text{Gram-Schmidt}(\nabla f(\theta_1),...,\nabla f(\theta_r))$ 
    
    \While{$\Vert\hat\nabla f(\theta)\Vert \geq \sqrt{2\mu\epsilon}$}{
        
        $\hat{\nabla} f(\theta) \gets \sum_{k=1}^r \partial_{u_k}f(\theta)u_k$ 
        
        $\theta\gets\theta-\alpha \hat{\nabla} f(\theta) $
    }
}

\textbf{return} $\theta$

\end{algorithm}

\begin{algorithm}[H]
\caption{Adaptive Low-Rank Gradient Descent 
}\label{alg: adaptive LRGD}

\textbf{Require:} $\epsilon > 0$, $\alpha \in \big(0,\frac{1}{L}\big]$, $\mu > 0$, $\theta_0 \in \R^p$

$\theta \gets \theta_0$

$\ccalU \gets \big\{\frac{1}{\Vert\nabla f(\theta)\Vert}\nabla f(\theta)\big\}$

$r \gets 1$ 

\While{$r \leq p$}{

\While{$\Vert\hat\nabla f(\theta)\Vert \geq \sqrt{2\mu\epsilon}$}{

$\hat{\nabla} f(\theta) \gets \sum_{u \in \ccalU} \partial_{u}f(\theta) u$

$\theta\gets\theta-\alpha \hat{\nabla} f(\theta) $
}

$u' \gets \nabla f(\theta) - \sum_{u \in \ccalU} \partial_{u}f(\theta) u$

$\ccalU \gets \ccalU \cup \{u'\}$

$r \gets r+1$
}
\textbf{return} $\theta$

\end{algorithm}


\section{Calculus of low-rank functions}

In this appendix, we describe the effect of standard algebraic operations, e.g., scalar multiplication, addition, pointwise product, and composition, on exact low-rank functions. We then close the section with a simple example of how one might monotonically transform functions to reduce their rank.

\subsection{Algebra of low-rank functions}
\label{Algebra of Low Rank Functions}

Recall that a differentiable function $f:\Theta \rightarrow \Rbb$ for $\Theta \subseteq \R^p$ is said to have rank $r \leq p$ if there exists a subspace $H \subseteq \R^p$ of dimension $r$ such that $\nabla f(\theta) \in H$ for all $\theta \in \Theta$ (and no subspace of smaller dimension satisfies this condition). The ensuing proposition presents how the rank is affected by various algebraic operations.

\begin{proposition}[Algebra of low-rank functions]
\label{prop: algebra of LR}
Let $f_1:\Theta \rightarrow \Rbb$ and $f_2:\Theta \rightarrow \Rbb$ be any two differentiable functions with ranks $r_1$ and $r_2$, respectively. Then, the following are true:
\begin{enumerate}[label=\arabic*),leftmargin=*,noitemsep,nolistsep]
\item For any constant $\alpha \in \R\backslash\{0\}$, the function $\alpha f_1 : \Theta \rightarrow \Rbb$, $(\alpha f_1)(\theta) = \alpha f_1(\theta)$ has rank $r_1$.
\item The function $f_1 + f_2 : \Theta \rightarrow \Rbb$, $(f_1 + f_2)(\theta) = f_1(\theta) + f_2(\theta)$ has rank at most $r_1 + r_2$.
\item The function $f_1 f_2 : \Theta \rightarrow \Rbb$, $(f_1 f_2)(\theta) = f_1(\theta) f_2(\theta)$ has rank at most $r_1 + r_2$.
\item For any differentiable function $g : \R \rightarrow \R$, the function $g \circ f_1 : \Theta \rightarrow \Rbb$, $(g \circ f_1)(\theta) = g(f_1(\theta))$ has rank at most $r_1$. For example, $\exp(f_1): \Theta \rightarrow \Rbb$, $\exp(f_1)(\theta) = \exp(f_1(\theta))$ has rank $r_1$.
\end{enumerate}
\end{proposition}

\begin{proof}
These results are straightforward consequences of basic facts in linear algebra. Let $H_1$ and $H_2$ be subspaces of dimension $r_1$ and $r_2$, respectively, such that $\nabla f_1(\theta) \in H_1$ and $\nabla f_2(\theta) \in H_2$ for all $\theta \in \Theta$. Then, the aforementioned parts can be established as follows:
\begin{enumerate}[label=\arabic*),leftmargin=*,noitemsep,nolistsep]
\item For any constant $\alpha \in \R$, we have $\alpha \nabla f_1(\theta) \in H_1$ for all $\theta \in \Theta$. Moreover, since $\alpha \neq 0$, the rank will not decrease.
\item For all $\theta \in \Theta$, $\nabla f_1(\theta) + \nabla f_2(\theta) \in H_1 + H_2$, where $H_1 + H_2$ is the subspace sum of $H_1$ and $H_2$. Moreover, the dimension of $H_1 + H_2$ is at most $r_1 + r_2$.
\item For all $\theta \in \Theta$, $f_1(\theta) \nabla f_2(\theta) + f_2(\theta) \nabla f_1(\theta) \in H_1 + H_2$ using the product rule. As before, the dimension of $H_1 + H_2$ is at most $r_1 + r_2$.
\item For all $\theta \in \Theta$, $g^{\prime}(f_1(\theta)) \nabla f_1(\theta) \in H_1$. In the case of the exponential function, $g^{\prime}(f_1(\theta)) = \exp(f_1(\theta)) > 0$, which means that the rank will not decrease.
\end{enumerate}
This completes the proof.
\end{proof}

We remark that the first three properties demonstrate how rank behaves with the usual algebra of functions. Moreover, the above operations can be applied to simple ridge functions to construct more complex exact low-rank functions.

\subsection{Example: Transformation to a low-rank function}
\label{Example: Transformation to a Low Rank Function}

Consider the function $f:(0,\infty)^p \rightarrow \Rbb$ defined by
$$ f(\theta) = \prod_{i = 1}^{p}{\theta_i} \, , $$
where $\theta = (\theta_1,\dots,\theta_p)$. Objective functions of this form are often optimized in isoperimetric problems, e.g., when volume is extremized with constrained surface area. Observe that 
$$ \partial_j f(\theta) = \prod_{i\neq j}\theta_i \, , $$
where $\partial_j$ denotes the partial derivative with respect to the $j$th coordinate for $j \in [p]$. Consequently, for any $i \in [p]$, $\nabla f(\mathbf{1}-e_i) = e_i$, where $\mathbf{1} = (1,\dots,1)\in \Rbb^p$ and $e_i \in \Rbb^p$ is the $i$th standard basis vector. Therefore, the function $f$ has rank $p$ using Definition \ref{definition: excat LR}.

Now, using the change-of-variables $\phi_i = \log(\theta_i)$ for every $i \in [p]$, define the transformed function $g:\Rbb^p \rightarrow \R$ as 
$$ g(\phi) = \log(f(\theta)) = \sum_{i = 1}^{p}{\phi_i} = \langle \mathbf{1},\phi \rangle \, , $$
where $\phi = (\phi_1,\dots,\phi_p)$. The function $g$ has rank $1$ using Proposition \ref{prop: exact LR conditions}.

Since $\log(\cdot)$ is a monotonic transformation, minimizing (or maximizing) $f(\theta)$ is equivalent to minimizing (or maximizing) $g(\phi)$. However, $g$ has rank $1$ while $f$ has rank $p$. Hence, this example demonstrates that it is sometimes possible to transform optimization problems where the objective function has high-rank into equivalent problems where the objective function has low-rank. We note that this example is inspired by the classical transformation of a \emph{geometric program} into a convex program \cite[Section 4.5]{boyd2004convex}.


\section{Proofs on approximate low-rank functions}

\subsection{Proof of Proposition \ref{prop: exact LR conditions}} \label{proof: prop exact LR}

$1)\Rightarrow 2)$: Let $\theta \in \Theta$ and $u\in H^\perp$. Consider $\gamma : t\in [0,1] \rightarrow f(\theta + tu)$. Note that its derivative satisfy $\gamma'(t) = \langle \nabla f(\theta + tu) , u \rangle =0$. Therefore, $\gamma(1) = \gamma (0)$, i.e. $f(\theta) = f(\theta+u)$.

$2) \Rightarrow 3)$: Take $e_1, \cdots, e_r$ an orthonormal basis of $H$, take $A = [e_1, \cdots , e_r]^{\top} \in \Rbb^{r\times p}$ and for all $x\in \Rbb^r$ $\sigma(x) = f(A^{\top}x)$. Then $A$ is a projection of $\Theta$ on $H$ and $\sigma$ is the restriction of $f$ to $H$.

$3) \Rightarrow 1)$: Assuming $f(\theta) = \sigma(A\theta)$, we get by the chain rule  \begin{equation}
    \nabla f(\theta) = A^{\top}\nabla \sigma(A\theta),
\end{equation}
as a consequence, $\nabla f(\theta) \in \text{span}(A)$.
\qed
\subsection{Proof of Proposition \ref{prop: apprx LR}} \label{proof: prop apprx LR}

We write Taylor's expansion around $\theta^*$,
\begin{equation}
\nabla f(\theta) = \nabla^2f(\theta^*) (\theta-\theta^*)+ o(\Vert\theta-\theta^*\Vert),
\end{equation}
where $\nabla^2f(\theta^*)$ is invertible, and the little-$o$ notation applies entry-wise. For any given subspace $H$, we can consider $u = \nabla^2f(\theta^*)^{-1} v$ where $v\in H^\perp$ and note that 
\begin{align}
    \nabla f(\theta^*+ tu) = tv + o(t),\\
    \Pi_H(\nabla f(\theta^*+ tu)) = o(t).
\end{align}
As a consequence, as $t\rightarrow 0$,
\begin{equation}
    \frac{\Vert\nabla f(\theta^*+tu) - \Pi_H(\nabla f(\theta^*+tu)) \Vert}{\Vert\nabla f(\theta^*+tu)\Vert} \rightarrow 1.
\end{equation}
\qed 

We remark that by looking at the rate of convergence in the equation above (which requires third derivatives), we could obtain a sharper lower bound for $\epsilon$.

\subsection{Proof of Proposition \ref{prop: gradients_span_directions}} \label{proof: prop gradients_span_directions}

The theorem is a consequence of the following lemma.

\begin{lemma}[Gradients on level sets]\label{lemma: all_directions}
For any $0<\lambda<\frac{\mu\Delta^2}{2}$ and any vector $u$ of unit norm, there exists $\theta$ s.t.
\begin{equation}
    f(\theta) = f(\theta^*)+\lambda \quad \text{and} \quad \frac{\nabla f(\theta)}{\Vert\nabla f(\theta)\Vert}=u, \quad \text{where} \quad \|\theta-\theta^*\|<\Delta.
\end{equation}
\end{lemma}
\begin{proof}
Let $0<\lambda<\frac{\mu\Delta^2}{2}$ note that the set,
\begin{equation}
    \Lambda = \{\theta \in \Theta : f(\theta) \leq f(\theta^*)+\lambda\}
\end{equation}
is a closed, convex subset of $\Rbb^p$, and is strictly included in $\ccalB(\theta^*,\Delta)$ because by strong convexity,
\begin{equation}
    \frac{\mu}{2}\Vert\theta-\theta^*\Vert^2 \leq f(\theta)-f(\theta^*), \quad \forall \theta\in \Theta.
\end{equation}
For a given vector of unit norm $u\in \Rbb^p$, consider the function $g(\theta) = \langle \theta , u \rangle \in\Rbb$. This function can be maximized on $\Lambda$ and the maximum is attained on some $\theta_u\in \Lambda$. Note that $\theta_u$ belongs to the boundary of $\Lambda$, because it results from the optimization of a linear function on a convex set, as a consequence: $\theta_u\neq \theta^*$. Now define $v= \frac{\nabla f(\theta_u)}{\Vert\nabla f(\theta_u)\Vert}$, let's assume by contradiction that $v \neq u$ and define $w = u-v$. By Taylor's expansion,
\begin{align}
    f(\theta_u+tw) &= f(\theta_u)+t \langle \nabla f(\theta_u), w \rangle +o(t),\\
    g(\theta_u+tw) &= g(\theta_u) + t \langle u, w\rangle.
\end{align}
Note that by definition $\langle \nabla f(\theta_u), w \rangle < 0$ while $\langle u, w \rangle > 0$. As a consequence, the maximality of $\theta_u$ is contradicted. Indeed, there exists $t>0$ and $h=tu$ such that
\begin{equation}
    \theta_u+h\in \Lambda, \quad \text{and} \quad g(\theta_u+h)\geq g(\theta_u).
\end{equation}
\end{proof}

Using this lemma, we proceed to the proof of Proposition \ref{prop: gradients_span_directions}. Let $H$ be a vector space of dimension $r<p$ such that the low rank approximation $\hat{\nabla}f$ is defined as the projection on this vector space,
\begin{equation}
    \hat{\nabla}f(\theta)= \Pi_H(\nabla f(\theta))\quad \forall \theta\in \Theta.
\end{equation}
Take any direction $u\in H^\perp$, and apply Lemma \ref{lemma: all_directions} with $\lambda = \frac{\mu\Delta^2}{2}$ which gives $\theta_u\in\Theta$ satisfying
\begin{equation}
    f(\theta_u) = f(\theta^*)+\lambda \quad \text{and} \quad \frac{\nabla f(\theta_u)}{\Vert\nabla f(\theta_u)\Vert}=u.
\end{equation}
By orthogonality, $\Pi_H(\nabla f(\theta_u))=0$ therefore,
\begin{equation}
    \Vert\nabla f(\theta_u) - \Pi_H(\nabla f(\theta_u))\Vert = \Vert\nabla f(\theta_u)\Vert.
\end{equation}
Recall from \cite{boyd2004convex} that for all $\theta\in\Theta$,
\begin{equation*}
    2\mu(f(\theta)-f(\theta^*)) \leq \Vert\nabla f(\theta)\Vert^2.
\end{equation*}
In particular
\begin{align*}
    2\mu\lambda \leq \Vert\nabla f(\theta)\Vert^2,
    \quad \text{and} \quad
    \mu\Delta \leq \Vert\nabla f(\theta)\Vert.
\end{align*}
\qed

\section{Proofs of oracle complexity analysis}

\subsection{Strongly convex setting} 

\subsubsection{Proof of Theorem \ref{thm: exact LR SC}} \label{sec: exact LR SC proof}

The proof is essentially based on the convergence of {\gd} for strongly convex loss functions. More precisely, when the objective function is exactly low-rank (Definition \ref{definition: excat LR}), then the iterates of {\lrgd} are indeed identical to those of {\gd}. However, {\lrgd} requires only $r$ oracle calls to directional derivatives of $f$ at each iteration. Therefore, we employ the well-known convergence of {\gd} for strongly convex objectives to find the total number of iterations required to reach an $\epsilon$-minimizer of the objective. However, as stated in the theorem's assumption, $f$ itself can not be strongly convex (as elaborated in Section \ref{sec: setup} as well). Rather, we assume that $f$ restricted to $H=\mathrm{span}(U)$ is $\mu$-strongly convex. Let us make this more clear by introducing a few notations. We denote by $F: \reals^r \to \reals$ the restriction of $f$ to $H$. That is,
\begin{align} \label{eq: restriced F}
    F(\omega)
    \coloneqq
    f\big(U \omega \big),
    \quad
    \forall \omega \in \reals^r.
\end{align}
Next, we relate the iterates of {\lrgd} on $f$ to the ones of {\gd} on $F$. Let us assume that {\lrgd} is initialized with $\theta_0 = 0$ (The proof can be easily generalized to arbitrary initialization). For $t=0,1,2,\cdots$ and stepsize $\alpha$, iterates of {\lrgd} are as follows
\begin{align} 
    \theta_{t+1} 
    =
    \theta_t - \alpha \grhat f(\theta_t)
    =
    \theta_t - \alpha \gr f(\theta_t),
\end{align}
which are identical to the iterates of {\gd} on $f$ due to the exact low-rankness assumption on $f$. Note that since $f$ is \emph{not} strongly convex, we may not utilize the linear convergence results on $f$. Nevertheless, we show in the following that such linear convergence still holds by using the restricted function $F$ defined above.

Let us define the {\gd} iterates $\omega_t$ on $F$ as follows
\begin{align}
    \omega_0 = U^\top \theta_0,
    \quad
    \omega_{t+1} 
    =
    \omega_t - \alpha \gr F(\omega_t),
    \quad
    t=0,1,2,\cdots .
\end{align}
Next, we show that the iterates $\{\theta_t\}$ and $\{\omega_t\}$ are related as $\omega_t = U^\top \theta_t$ for any $t=0,1,2,\cdots$. By definition, $\omega_0 = U^\top \theta_0$. Assume that for any $0 \leq k \leq t$ we have that $\omega_k = U^\top \theta_k$. We can write
\begin{align}
    \omega_{t+1} 
    =
    \omega_t - \alpha \gr F(\omega_t)
    =
    U^\top \theta_t - \alpha U^\top \gr f(U \omega_t)
    =
    U^\top \left(\theta_t - \alpha \gr f(\theta_t) \right)
    =
    U^\top \theta_{t+1}.
\end{align}
In above, we used the fact that $U \omega_t = \theta_t$ for any $t$. To verify this, note that $U \omega_t = U U^\top \theta_t$. On the other hand, one can verify that all the iterates $\{\theta_t\}$ live in $H$. Therefore, the projection of $\theta_t$ to $H$ equals to itself, i.e., $U \omega_t = U U^\top \theta_t = \theta_t$.

Next, we show that the minimum value of the two objectives $f$ and $F$ are indeed identical. Intuitively, $F$ is the restriction of $f$ to $H$ and we know that the function value of $f$ does not change along the directions of $H^\perp$. Therefore, $f$ and $F$ attain the same minimum function values. To show this concretely, first note that $F^* \coloneqq \min_{\omega \in \reals^r} F(\omega) \geq f^* \coloneqq \min_{\theta \in \reals^p} f(\theta)$. Let $\theta^*$ be a minimizer of $f$, i.e., $f(\theta^*) = f^*$. Next, we have that
\begin{align}
    F(U^\top \theta^*)
    =
    f(U U^\top \theta^*)
    =
    f(\theta^* - U_{\perp} U_{\perp}^\top \theta^*)
    =
    f(\theta^*)
    =
    f^*,
\end{align}
which yields that $F^* = f^*$. In above, $U_{\perp} U_{\perp}^\top$ denotes the projection to $H^{\perp}$ and we used the fact that $U_{\perp} U_{\perp}^\top \theta^* \in H^{\perp}$ and from Proposition \ref{prop: exact LR conditions}, $f(\theta^* - U_{\perp} U_{\perp}^\top \theta^*) = f(\theta^*)$.

Now, we are ready to derive the convergence of $\{\theta_t\}$ on $f$. Since $F$ is $\mu$-strongly convex and $L$-smooth, the following contraction holds for {\gd} iterates $\{\omega_t\}$ on $F$ for stepsize $\alpha=1/L$
\begin{align} 
    F(\omega_t) - F^*
    \leq
    \exp(-t/\kappa) \left( F(\omega_0) - F^*\right).
\end{align}
Substituting $F^* = f^*$ and $F(\omega_t) = f(\theta_t)$ in above, we have that
\begin{align} 
    f(\theta_t) - f^*
    \leq
    \exp(-t/\kappa) \left( f(\theta_0) - f^*\right).
\end{align}
Therefore, {\lrgd} takes $T = \kappa \log(\Delta_0/\epsilon)$ iterations to reach an $\epsilon$-minimizer of $f$ where $\Delta_0 = f(\theta_0) - f^*$. Each iteration of {\lrgd} requires $r$ calls to directional derivatives of $f$. Together with the oracle complexity of the initial $r$ full gradient computations (hence $pr$ directional derivatives), the total oracle complexity sums up to be
\begin{align}
\ccalC_{{\normalfont{\lrgd}}}
=
T r + p r
=
\kappa r \log(\Delta_0/\epsilon) + p r.
\end{align}

In above, we used the well-known convergence result for strongly convex functions.
For the reader's convenience, we reproduce this convergence rate here as well. The next lemma characterizes the convergence of {\gd} for strongly convex objectives \cite{Nesterov2004}.

\begin{lemma}[Convergence in strongly convex case \cite{Nesterov2004}] \label{lemma: GD SC convergence}
Let $F$ be $L$-smooth and $\mu$-strongly convex according to Assumption \ref{assumption: smooth} and \ref{assumption: SC} with $\kappa = L/\mu$ denote the condition number. Then, for stepsize $\alpha=1/L$, the iterates of {\normalfont{\gd}} converge to the unique global minimizer as follows
\begin{align}
F(\theta_t) - F^*
\leq
\exp (-t/\kappa ) \left(F(\theta_0) - F^*\right).
\end{align}
\end{lemma}
\qed

\subsubsection{Proof of Theorem \ref{thm: apprx LR SC}} \label{sec: apprx LR SC proof}

The proof consists of two main parts. First, we characterize the gradient inexactness induced by the gradient approximation step of {\lrgd} which will be used in the second part to derive convergence rates for inexact gradient descent type methods.
\begin{lemma}[Gradient approximation] \label{lemma: LRGD gr apprx error}
Assume that the $f$ is $(\eta,\epsilon)$-approximately rank-$r$ according to Definition \ref{definition: apprx LR} and the following condition holds 
\begin{align} \label{eq: sigma_r condition}
\frac{2r}{\sigma_r} \left(\eta + \frac{\epsilon}{\epsilon'} \right)
<
\frac{1}{\sqrt{10}},
\end{align}
where $\sigma_r$ is the smallest singular value of the matrix $G = [g_1/{\Vert g_1\Vert},\cdots,g_r/{\Vert g_r\Vert}]$ with $g_j \coloneqq \gr f(\theta^j)$ and $\Vert g_j \Vert > \epsilon'$ for all $j \in [r]$. Then, the gradient approximation error of {\normalfont{\lrgd}} is bounded as follows
\begin{align}
    \Vert \hat{\nabla}f(\theta) - \nabla f(\theta) \Vert 
    \leq
    \tilde{\eta} \Vert \nabla  f(\theta) \Vert + \epsilon, 
\end{align}
for $\tilde{\eta} \coloneqq  \eta + {2r(\eta + \frac{\epsilon}{\epsilon'})}/{\sigma_r}$.
\end{lemma}
\begin{proof}
Let $G = U \Sigma V^\top$ and $G_H = U_H \Sigma_H V_H^{\top}$ respectively denote SVDs for matrices $G$ and $G_H$, where
\begin{gather}
G = \begin{bmatrix}
\frac{1}{\Vert g_1\Vert}g_1, \dots, \frac{1}{\Vert g_r\Vert}g_r
\end{bmatrix}, \quad
G_H = \begin{bmatrix}
\frac{1}{\Vert g_1\Vert}\Pi_H(g_1), \dots, \frac{1}{\Vert g_r\Vert}\Pi_H(g_r)
\end{bmatrix}.
\end{gather}
Since $f$ is $(\eta,\epsilon)$-approximately rank-$r$ according to Definition \ref{definition: apprx LR}, we can bound the operator norm of $G - G_H$ as follows
\begin{align} \label{eq: G - G_H bound}
    \Vert G - G_H\Vert_{\op} 
    \leq
    \sqrt{r}\Vert G - G_H\Vert_{\mathrm{F}}
    \leq
    r\max_{j \in [r]}\frac{1}{\Vert g_j \Vert} \Vert g_j - \Pi_H(g_j) \Vert
    \leq
    r \left(\eta + \frac{\epsilon}{\epsilon'}\right),
\end{align}
where in the last inequality we used the assumption that $\Vert g_j \Vert > \epsilon'$ for all $j \in [r]$. The condition stated in the lemma ensures that $G$ is full-rank (since $\sigma_r > 0$). Next, we show that projecting columns of $G$ on $H$ does not degenerate its rank, that is, $G_H$ is full-rank as well. To do see, we employ Weyl's inequality \cite[Corollary
7.3.5(a)]{HornJohnson2013} to write that
\begin{align}
    \sigma_r(G_H)
    \geq
    \sigma_r - \sigma_1(G - G_H)
    \geq
    \sigma_r - \Vert G - G_H\Vert_{\op} 
    \overset{(a)}{\geq}
    \sigma_r - r \left(\eta + \frac{\epsilon}{\epsilon'}\right)
    \overset{(b)}{>}
    0.
\end{align}
In above, $\sigma_r(G_H)$ and $\sigma_1(G - G_H)$ respectively denote the smallest and largest singular values of $G_H$ and $G- G_H$, $(a)$ follows from the bound in \eqref{eq: G - G_H bound}, and $(b)$ is an immediate result of the condition \eqref{eq: sigma_r condition}. Since $H$ is a subspace of rank $r$ and columns of $G_H$ are all in $H$, we have that $H = \mathrm{span}(U_H)$. Next, we employ Wedin's $\sin\Theta$ theorem \cite{chen2020spectral} to conclude that
\begin{align}
    \text{Sin}\Theta(\mathrm{span}(U),\mathrm{span}(U_H)) = \Vert \Pi_U - \Pi_{H}\Vert_{\op} 
    \leq
    \frac{2}{\sigma_r} \Vert G - G^*\Vert_{\op} 
    \leq
    \frac{2r}{\sigma_r} \left( \eta + \frac{\epsilon}{\epsilon'} \right).
\end{align}
This yields that for any $ \theta\in \reals^p$ we can uniformly bound the {\lrgd} gradient approximation error as follows
\begin{align}
    \Vert \hat{\nabla}f(\theta) - \nabla f(\theta) \Vert 
    &=
    \Vert \Pi_U(\nabla f(\theta)) - \nabla f(\theta)\Vert  \\
    &\leq
    \Vert \Pi_H(\nabla f(\theta)) - \nabla f(\theta) \Vert 
    +
    \Vert \Pi_H(\nabla f(\theta)) - \Pi_U(\nabla f(\theta)) \Vert \\
    &\leq
    \tilde{\eta} \Vert \nabla  f(\theta) \Vert + \epsilon, 
\end{align}
for $\tilde{\eta} \coloneqq  \eta + {2r(\eta + \frac{\epsilon}{\epsilon'})}/{\sigma_r}$.
\end{proof}

Next, we characterize convergence of inexact gradient descent and use {\lrgd}'s gradient approximation error bound derived in Lemma \ref{lemma: LRGD gr apprx error} to complete the proof. 

Let us get back to the setting of Theorem \ref{thm: apprx LR SC}. From smoothness of $f$ in Assumption \ref{assumption: smooth} and the update rule of {\lrgd}, i.e. $\theta_{t+1} = \theta_t - \alpha \grhat f(\theta_t)$, we can write
\begin{align} 
    f(\theta_{t+1})
    &=
    f(\theta_t - \alpha \grhat f(\theta_t)) \\
    &\leq
    f(\theta_t)
    -
    \alpha \langle \gr f(\theta_t), \grhat f(\theta_t) \rangle
    +
    \frac{L}{2} \alpha^2 \Vert \grhat f(\theta_t) \Vert^2\\
    &=
    f(\theta_t)
    -
    \alpha \Vert \gr f(\theta_t) \Vert^2
    -
    \alpha \langle \gr f(\theta_t), e_t \rangle
    +
    \frac{L}{2} \alpha^2 \Vert \grhat f(\theta_t) \Vert^2, \label{eq: LRGD contraction 1}
\end{align}
where $e_t \coloneqq \grhat f(\theta_t) - \gr f(\theta_t)$ denotes the gradient approximation error at iteration $t$. Next, we employ the gradient approximation error bound derived in Lemma \ref{lemma: LRGD gr apprx error}, that is $\Vert e_t \Vert \leq \tilde{\eta} \Vert \nabla  f(\theta_t) \Vert + \tilde{\epsilon}$, where we denote $\tilde{\epsilon} \coloneqq \sqrt{\mu \epsilon/5}$ and $\tilde{\eta} \coloneqq  \eta + {2r(\eta + \frac{\tilde{\epsilon}}{\epsilon'})}/{\sigma_r}$ for notation simplicity. Together with simple algebra from \eqref{eq: LRGD contraction 1}, we have that
\begin{align} 
    f(\theta_{t+1})
    &\leq
    f(\theta_t)
    -
    \alpha \Vert \gr f(\theta_t) \Vert^2
    +
    \frac{\alpha}2{}\Vert \gr f(\theta_t) \Vert^2
    +
    \frac{\alpha}2{}\Vert e_t \Vert^2
    +
    \alpha^2 L \Vert \gr f(\theta_t) \Vert^2
    +
    \alpha^2 L \Vert e_t \Vert^2 \\
    &\leq
    f(\theta_t)
    -
    \frac{\alpha}{2} \left(1 - 2 \alpha L - 2 \tilde{\eta}^2 (1 + 2 \alpha L) \right) \Vert \gr f(\theta_t) \Vert^2
    +
    \alpha (1 + 2 \alpha L) \tilde{\epsilon}^2. \label{eq: LRGD contraction 2}
'r\end{align}
Next, we use the gradient dominant property of strongly convex loss functions, that is $\Vert \gr f(\theta) \Vert^2 \geq 2 \mu (f(\theta) - f^*)$, $\forall \theta$, which together with \eqref{eq: LRGD contraction 2} implies that
\begin{align} 
    f(\theta_{t+1}) - f^*
    &\leq
    \left( 1 - \tilde{\alpha} \mu  \right)
    \left(f(\theta_t) - f^* \right)
    +
    \alpha (1 + 2 \alpha L) \tilde{\epsilon}^2,
\end{align}
for the shorthand notation $\tilde{\alpha} \coloneqq \alpha \left(1 - 2 \alpha L - 2 \tilde{\eta}^2 (1 + 2 \alpha L) \right)$. Now we pick the stepsize $\alpha = \frac{1}{8L}$ and use the condition $\tilde{\eta} \leq 1/\sqrt{10}$ to conclude that
\begin{align} \label{eq: LRGD contraction 3}
    f(\theta_{t+1}) - f^*
    &\leq
    \left( 1 - \frac{1}{16 \kappa}  \right)
    \left(f(\theta_t) - f^* \right)
    +
    \frac{5}{32L} \tilde{\epsilon}^2. 
\end{align}
The contraction bound in \eqref{eq: LRGD contraction 3} implies that after $T$ iterations of {\lrgd} the final suboptimality is bounded as below
\begin{align} 
    f(\theta_{T}) - f^*
    &\leq
    \exp\left( - \frac{T}{16 \kappa}  \right)
    \left(f(\theta_0) - f^* \right)
    +
    \frac{\epsilon}{2}.
\end{align}
Finally, in order to reach an $\epsilon$-minimizer of $f$, it suffices to satisfy the following condition
\begin{align} 
    \exp\left( - \frac{T}{16 \kappa}  \right)
    \left(f(\theta_0) - f^* \right)
    \leq
    \frac{\epsilon}{2},
\end{align}
yielding that {\lrgd} runs for $T = 16 \kappa \log(2 \Delta_0 / \epsilon)$ iterations. Thus, the total oracle complexity of {\lrgd} would be
\begin{align}
\ccalC_{{\normalfont{\lrgd}}}
=
T r + p r
=
16 \kappa r \log(2 \Delta_0 / \epsilon) + pr,
\end{align}
as claimed in Theorem \ref{thm: apprx LR SC}.
\qed

\subsection{Non-convex setting} 

\subsubsection{Proof of Theorem \ref{thm: exact LR NC}} \label{sec: exact LR NC proof}

We first note that when $f$ is exactly rank-$r$, then the iterates of {\lrgd} are identical to (exact) {\gd}'s ones. Therefore, the number of iterations required by {\lrgd} to reach an $\epsilon$-stationary point is determined by the convergence rate of {\gd} for non-convex losses. Though this is well-known in the optimization community, it is rarely stated explicitly due to its simplicity. So, in the following, we provide such a convergence result for the readers' convenience.
\begin{lemma}[Convergence in non-convex case] \label{lemma: GD NC convergence}
Let $f$ be $L$-smooth 
according to Assumption \ref{assumption: smooth}. Moreover, assume that $f$ is exactly rank-$r$ according to Definition \ref{definition: apprx LR}. Then, for stepsize $\alpha=1/L$ and after $T$ iterations of {\lrgd} in Algorithm \ref{alg: LRGD}, there exists an iterate $0 \leq t \leq T-1$ for which
\begin{align} 
    \big\Vert \gr f(\theta_t) \big\Vert^2 
    \leq
    \frac{2L}{T} \left(f(\theta_0) - f^* \right).
\end{align}
\end{lemma}
\begin{proof}
Given the initialization $\theta_0$ and stepsize $\alpha$, the iterates of {\lrgd} (or {\gd}) can be written as $\theta_{t+1} = \theta_t - \alpha \gr f(\theta)$ for any $t=0,1,2,\cdots$. Since $f$ is $L$-smooth, for stepsize $\alpha=1/L$ we have that
\begin{align} \label{eq: GD NC iterates}
    f(\theta_{t+1})
    &=
    f(\theta_t - \alpha \gr f(\theta_t)) \\
    &\leq
    f(\theta_t)
    -
    \alpha \langle \gr f(\theta_t), \gr f(\theta_t) \rangle
    +
    \frac{L}{2} \alpha^2 \Vert \gr f(\theta_t) \Vert^2\\
    &=
    f(\theta_t)
    -
    \frac{1}{2L} \Vert \gr f(\theta_t) \Vert^2.
\end{align}
Summing \eqref{eq: GD NC iterates} for all $t=0,1,\cdots,T-1$ yields that
\begin{align} 
    \min_{0 \leq t \leq T-1} \big\Vert \gr f(\theta_t) \big\Vert^2 
    \leq
    \frac{1}{T} \sum_{t=0}^{T-1} \Vert \gr f(\theta_t) \Vert^2
    \leq
    \frac{2L}{T} (f(\theta_0 - f^*),
\end{align}
concluding the claim.
\end{proof}

Getting back to Theorem \ref{thm: exact LR NC} and using Lemma \ref{lemma: GD NC convergence}, it is straightforward to see that {\lrgd} requires $T = 2L \Delta_0/\epsilon$ iterations to reach an $\epsilon$-stationary point of $f$ which yields the total oracle complexity of {\lrgd} to be
\begin{align}
\ccalC_{{\normalfont{\lrgd}}}
=
T r + p r
=
\frac{2 r L \Delta_0}{\epsilon^2} + p r,
\end{align}
where $pr$ is indeed the (directional) gradient computation cost associated with the $r$ gradient vectors $g_1, \cdots, g_r$ require to construct thee matrix $G$ in {\lrgd} (Algorithm \ref{alg: LRGD}).
\qed

\subsubsection{Proof of Theorem \ref{thm: apprx LR NC}} \label{sec: apprx LR NC proof}

We first invoke Lemma \ref{lemma: LRGD gr apprx error} which bounds the gradient approximation error of {\lrgd}. Note that this lemma does not require any convexity assumptions, hence we use it in the proof of both strongly convex (Theorem \ref{thm: apprx LR SC}) and non-convex settings (Theorem \ref{thm: apprx LR NC}).

First, note that from Lemma \ref{lemma: LRGD gr apprx error} and given the condition stated in Theorem \ref{thm: apprx LR NC}, the gradient approximation error (inexactness) is upper-bounded as $\Vert e_t \Vert \leq \tilde{\eta} \Vert \nabla  f(\theta_t) \Vert + \tilde{\epsilon}$ for $\tilde{\eta} \coloneqq  \eta + {2r(\eta + \frac{\tilde{\epsilon}}{\epsilon'})}/{\sigma_r}$ and $\tilde{\epsilon} \coloneqq  \epsilon/3$. Secondly, we characterize the convergence of inexact {\gd} (which {\lrgd} updates can be viewed as such) to derive {\lrgd}'s oracle complexity to reach an stationary point.

Let us borrow the contraction bound \eqref{eq: LRGD contraction 2} from Theorem \ref{thm: apprx LR SC}. Note that we used only the smoothness of $f$ (and not the strong convexity) to get this bound, hence, we can safely employ it here for the smooth and nonconvex objectives. That is, for stepsize $\alpha$, the iterates of {\lrgd} satisfy the fallowing contraction
\begin{align} 
    f(\theta_{t+1})
    &\leq
    f(\theta_t)
    -
    \frac{\alpha}{2} \left(1 - 2 \alpha L - 2 \tilde{\eta}^2 (1 + 2 \alpha L) \right) \Vert \gr f(\theta_t) \Vert^2
    +
    \alpha (1 + 2 \alpha L) \tilde{\epsilon}^2.
\end{align}
Picking the stepsize $\alpha = \frac{1}{8L}$ and using the condition $\tilde{\eta} \leq 1/\sqrt{10}$, we have from the above inequality that
\begin{align} 
    f(\theta_{t+1})
    &\leq
    f(\theta_t)
    -
    \frac{1}{32L} \Vert \gr f(\theta_t) \Vert^2 + \frac{5}{32L} \tilde{\epsilon}^2.
\end{align}
Averaging the above progression over $t=0,1,\cdots,T-1$ yields that
\begin{align} 
    \min_{0 \leq t \leq T-1} \big\Vert \gr f(\theta_t) \big\Vert^2 
    \leq
    \frac{1}{T} \sum_{t=0}^{T-1} \Vert \gr f(\theta_t) \Vert^2
    \leq
    \frac{32L}{T} (f(\theta_0) - f^*)
    +
    \frac{5}{9} \epsilon^2.
\end{align}
Therefore, for $T = 72 L \Delta_0 / \epsilon^2$, we conclude that there exists an iteration $t \in \{0,\cdots,T-1\}$ for which
\begin{align} 
    \big\Vert \gr f(\theta_t) \big\Vert^2 
    \leq
    \frac{32L}{T} \Delta_0
    +
    \frac{5}{9} \epsilon^2
    =
    \epsilon^2,
\end{align}
which yields that the total oracle complexity of {\lrgd} is
\begin{align}
\ccalC_{{\normalfont{\lrgd}}}
=
T r + p r
=
\frac{72 r L \Delta_0}{\epsilon^2} + p r.
\end{align}
\qed


\section{Discussion on low-rank Hessians}\label{ap: discussion_low_rank}

As stated in the introduction, it has been noted in the literature that many functions of interest -- such as losses of neural networks -- appear to have approximate low rank Hessians, cf. \cite{gur2018gradient,Sagunetal2018,Papyan2019,Wuetal2021} and the references therein. In this paragraph, we will argue why this observation makes the approximate low-rank assumption realistic at a local scale. In particular, we aim to show the following result,

\begin{proposition}
Given a function $f:\Theta\rightarrow \Rbb$ with bounded third derivatives and some $\bar\theta\in \Theta$ denote by $\sigma_r$ the $r$-th singular value $\nabla^2 f(\bar\theta)$. There exists a linear subspace $H\subset\Theta$ of dimension $r$ that satisfies,
\begin{equation*}
    \|\nabla f(\theta)-\Pi_H(\nabla f(\theta))\|\leq M\|\theta-\thetab\|^2+\sigma_r||\theta-\thetab||.
\end{equation*}
for some constant $M \geq 0$ that only depends on the bound on the third derivatives of $f$.  
\end{proposition}

\begin{proof}
By Taylor's expansion on a function with bounded third derivatives, there exists some $M\geq 0$ such that, 
\begin{align}
    \forall \theta\in \Theta: \|\nabla f(\theta)- \nabla f(\thetab) - \nabla^2f(\thetab) (\theta-\thetab)\| \leq M\|\theta-\thetab\|^2, \label{eq: taylor}.
\end{align}

If we denote by $\sigma_1,\cdots,\sigma_p$ (resp. $u_1,\cdots,u_p$, resp. $v_1,\cdots,v_p$) the singular values (resp. left singular vectors, resp. right singular vectors) obtained from the decomposition of $\nabla^2 f(\thetab)$, we observe that $y = \nabla f(\thetab) + \nabla^2f(\thetab) (\theta-\thetab)$ satisfies,
\begin{equation}
    \|y-\Pi_H(y)\|\leq \sigma_r||\theta-\thetab||, \label{eq : singular}
\end{equation}
where $H = \spn (\nabla f(\thetab),u_1,\dots,u_{r-1})$. Indeed, first note that $\Pi_H(\nabla f(\thetab) ) = \nabla f(\thetab)$ and second observe that for any $x\in \Rbb^p: \|x-\Pi_H(\nabla^2f(\thetab)x)\| = \|\sum_{i\geq r}\sigma_i\langle v_i , x\rangle u_i\| \leq \sigma_r \|x\|$.

Finally, it is a property of the projection operator that $\|\Pi_H(\nabla f(\theta)) - \Pi_H(y)\|\leq \|\nabla f(\theta) - y\|$. As a consequence we can combine inequality both equations \eqref{eq: taylor} and \eqref{eq : singular} by triangular inequality into,
\begin{equation*}
    \|\nabla f(\theta)-\Pi_H(\nabla f(\theta))\|\leq M\|\theta-\thetab\|^2+\sigma_r||\theta-\thetab||.
\end{equation*}
\end{proof}

The expression obtained above is very similar to the definition of  approximate low rankness. In fact, for a given $\eta>0$ the $\eta-$approximately low-rank assumption can hold as long as $M\|\theta-\thetab\|^2+\sigma_r||\theta-\thetab||\leq \eta \|\nabla f(\theta)\|$. Typically this happens when the following conditions are met (1) $\nabla^2 f(\thetab)$ is approximately low-rank, i.e., $\sigma_r$ is small; (2) $\|\nabla f(\theta)\|$ is much greater than $\|\theta-\thetab\|$; (3) The third derivative constant $M$ is very small. 

In practice, the approximately low-rank assumption must hold for multiple iterates of $\lrgd$ of the form $\theta' = \theta+\alpha \hat \nabla f(\theta)$ where $\alpha$ is the step size. It is thus natural to expect $\|\theta-\thetab\|\sim \alpha N\|\nabla f (\theta)\|$ where $N$ is the number of iterations for which the approximate low rank condition holds. This yields the following condition $\alpha MN\|\nabla f(\theta)\|+\sigma_r \leq \eta$, note that this condition may be satisfied for very small step sizes $\alpha$.

\section{Low-rank objective functions of neural networks}\label{ap: Low-rank objective functions of neural networks}

We next describe an experiment which illustrates that objective functions in neural network training problems can have (approximately) low-rank structure. This experiment is partly inspired by the simulations in \cite[Appendix B]{JadbabaieMakurShah2022}. Specifically, we use the standard MNIST database for multi-class classification \cite{LeCunCortesBurgesMNIST}, which consists of $60,000$ samples of labeled training data and $10,000$ samples of labeled test data. As is standard practice, we normalize and one-hot encode the datasets so that feature vectors (or images) belong to $[0,1]^{784}$ and labels belong to $\{0,1\}^{10}$. Using Keras, we then construct a $2$-layer fully connected neural network with $784$ neurons in the input layer, $784$ neurons in the hidden layer with rectified linear unit (ReLU) activation functions, and $10$ neurons in the output layer with softmax activation function. This neural network has $p = 623,290$ weights. Moreover, we use the cross-entropy loss function.

We train this neural network on the normalized and one-hot encoded training data using Adam \cite{KingmaBa2015} for $20$ epochs with batch size $200$. (Note that our validation accuracy on the test dataset is typically over $98\%$.) Let $\theta^* \in \R^p$ be the optimal network weights obtained from this training. For a fixed standard deviation parameter $\tau = 0.2$, we construct $n = 100$ independent random perturbations $\theta_1,\dots,\theta_n \in \R^p$ of the optimal network weights $\theta^*$ as follows:
\begin{equation}
\forall i \in [n], \enspace \theta_i = \theta^* + \tau Z_i \, ,
\end{equation}
where $Z_1,\dots,Z_n$ are independent and identically distributed Gaussian random vectors with zero mean and identity covariance. Note that $\tau$ is chosen to be the same order of magnitude as ``non-zero'' (i.e., larger) entries of $\theta^*$; $\tau$ provides ``soft'' control over the size of the local neighborhood considered around the optimum $\theta^*$. Then, for any randomly chosen batch of $200$ test data samples, we construct the matrix of gradients $G = [g_1, \dots, g_n] \in \R^{p \times n}$, whose $i$th column is given by
\begin{equation}
g_i = \nabla f(\theta_i)
\end{equation} 
for $i\in [n]$, where $\nabla f : \R^p \rightarrow \R^p$ denotes the gradient of the empirical cross-entropy risk defined by the $200$ test samples with respect to the network weights (i.e., $\nabla f$ is the sum of the gradients corresponding to the $200$ test samples). The gradients in $G$ are computed using standard automatic differentiation procedures in Keras.

Finally, let $\sigma_1(G) \geq \sigma_2(G) \geq \cdots \geq \sigma_n(G) \geq 0$ denote the ordered singular values of the matrix of gradients $G$, where $\|G\|_{\mathrm{F}}^2 = \sum_{i = 1}^{n}{\sigma_i(G)^2}$. Figure \ref{fig: low-rank-mnist} plots the normalized squared singular values $(\sigma_1(G)/\|G\|_{\mathrm{F}})^2 \geq (\sigma_2(G)/\|G\|_{\mathrm{F}})^2 \geq \cdots \geq (\sigma_n(G)/\|G\|_{\mathrm{F}})^2 \geq 0$ against the order indices $[n]$. It shows that $G$ is approximately low-rank with approximate rank around $10$, since the majority of the ``energy'' in $\|G\|_{\mathrm{F}}^2$ is captured by the $10$ largest singular values. So, despite there being $100$ different gradients in $G$ in a neighborhood of $\theta^*$, they are all approximately captured by a span of $10$ vectors in $\R^p$. Therefore, this simulation illustrates that the empirical cross-entropy risk (or objective function) corresponding to the MNIST neural network training task is approximately low-rank in a neighborhood of $\theta^*$. 

We remark that small perturbations of $\tau$,
increasing $n$, increasing the batch size of $200$ (which defines the empirical cross-entropy risk), and changing the hyper-parameters of training (e.g., number of neurons in hidden layer, choice of optimization algorithm, activation functions, etc.) do not seem to affect the qualitative nature of the plot in Figure \ref{fig: low-rank-mnist}. We refer readers to \cite{gur2018gradient,Sagunetal2018,Papyan2019,Wuetal2021} for more thorough empirical studies of such low-rank structure. (As discussed in Appendix \ref{ap: discussion_low_rank}, the low-rank Hessian structure analyzed in these works  \cite{gur2018gradient,Sagunetal2018,Papyan2019,Wuetal2021} yields low-rank gradient structure.) 

\begin{figure}[t]
\centering
\includegraphics[width=0.8\linewidth]{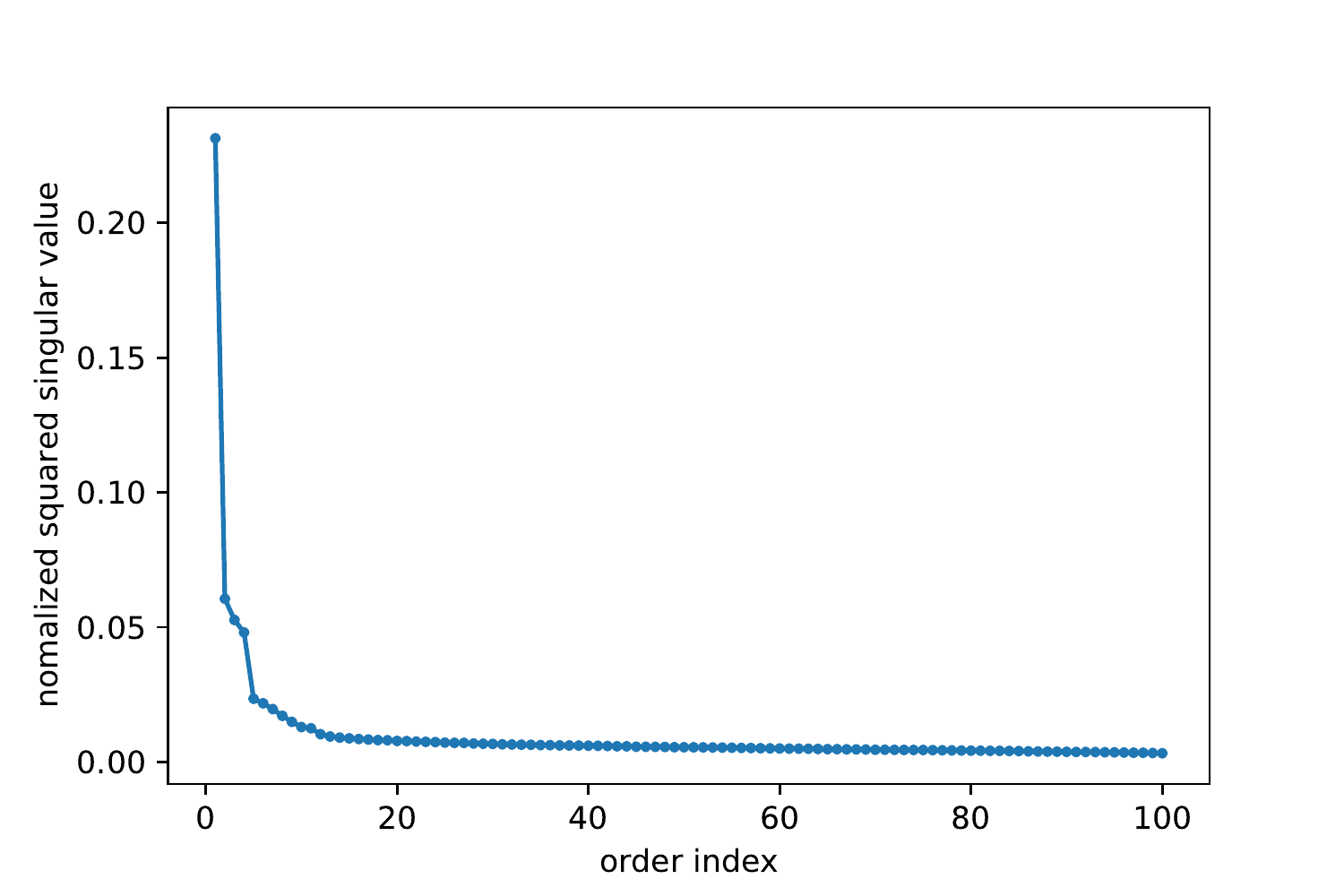}
\caption{Plot of squared singular values of the matrix of gradients $G$ normalized by $\|G\|_{\mathrm{F}}^2$. Note that the $i$th largest normalized squared singular value $(\sigma_i(G)/\|G\|_{\mathrm{F}})^2$ is displayed at order index $i \in [100]$.}  
\label{fig: low-rank-mnist}
\end{figure}


\section{Additional experimental tables}\label{ap: additional_experiment}

\begin{table}[h!]
  \centering
  \begin{tabular}{ccccc}
    \toprule
    \multicolumn{5}{c}{$H$}                   \\
    \cmidrule(r){2-5}
    $\theta_0$     
    & 
    $\begin{pmatrix}
    1 & 0\\
    0 & 1
    \end{pmatrix}$
    & 
    $\begin{pmatrix}
    10 & 0\\
    0 & 1
    \end{pmatrix}$
    &
    $\begin{pmatrix}
    100 & 0\\
    0 & 1
    \end{pmatrix}$
    &
    $\begin{pmatrix}
    1000 & 0\\
    0 & 1
    \end{pmatrix}$
    \\
    \midrule
    $\begin{pmatrix} 1.00 & 0.00 \end{pmatrix}$ & 6 (7)  & 12 (10) & 20 (14) & 26 (17)  \\
    $\begin{pmatrix} 0.87 & 0.50 \end{pmatrix}$ & 6 (7) & 46 (35) & 460 (293) & 4606 (3025)  \\
    $\begin{pmatrix} 0.71 & 0.71 \end{pmatrix}$ & 6 (7) & 60 (40) & 600 (385) & 6008 (3880)  \\
    $\begin{pmatrix} 0.50 & 0.87 \end{pmatrix}$ & 6 (7) & 68 (51) & 682 (443) & 6820 (4120)  \\
    $\begin{pmatrix} 0.00 & 1.00 \end{pmatrix}$ & 6 (7) & 72 (40) & 736 (372) & 7376 (3692)  \\
    \bottomrule
  \end{tabular}
  \vspace{2mm}
  \caption{Same as Table \ref{tab: table half} with $\alpha$ set to $1/(4L).$}
  \label{tab: table fourth}
\end{table}

\begin{table}[h!]
  \centering
  \begin{tabular}{ccccc}
    \toprule
    \multicolumn{5}{c}{$H$}                   \\
    \cmidrule(r){2-5}
    $\theta_0$     
    & 
    $\begin{pmatrix}
    1 & 0\\
    0 & 1
    \end{pmatrix}$
    & 
    $\begin{pmatrix}
    10 & 0\\
    0 & 1
    \end{pmatrix}$
    &
    $\begin{pmatrix}
    100 & 0\\
    0 & 1
    \end{pmatrix}$
    &
    $\begin{pmatrix}
    1000 & 0\\
    0 & 1
    \end{pmatrix}$
    \\
    \midrule
    $\begin{pmatrix} 1.00 & 0.00 \end{pmatrix}$ & 18 (13)  & 38 (23) & 58 (33) & 80 (44)  \\
    $\begin{pmatrix} 0.87 & 0.50 \end{pmatrix}$ & 18 (13) & 114 (89) & 1152 (685) & 11512 (7258)  \\
    $\begin{pmatrix} 0.71 & 0.71 \end{pmatrix}$ & 18 (13) & 150 (102) & 1502 (898) & 15018 (10331)  \\
    $\begin{pmatrix} 0.50 & 0.87 \end{pmatrix}$ & 18 (13) & 170 (108) & 1704 (1079) & 17052 (9520)  \\
    $\begin{pmatrix} 0.00 & 1.00 \end{pmatrix}$ & 18 (13) & 184 (96) & 1844 (926) & 18444 (9226)  \\
    \bottomrule
  \end{tabular}
  \vspace{2mm}
  \caption{Same as Table \ref{tab: table half} with $\alpha$ set to 
  $1/(10L)$.
  }
  \label{tab: table tenth}
\end{table}


\newpage 

\bibliographystyle{unsrt}
\bibliography{biblio}

\end{document}